\documentclass[11pt,a4paper]{article}

\usepackage[utf8]{inputenc}
\usepackage[T1]{fontenc}
\usepackage{lmodern}
\usepackage[margin=1in,headheight=14pt]{geometry}
\usepackage{microtype}
\usepackage{graphicx}
\usepackage{amsmath}
\usepackage{amssymb}
\usepackage{amsthm}
\usepackage{bm}
\usepackage{mathrsfs}
\usepackage{algorithm}
\usepackage{algpseudocode}
\usepackage{booktabs}
\usepackage{array}
\usepackage{multirow}
\usepackage{enumitem}
\usepackage{comment}
\usepackage[numbers,sort&compress]{natbib}
\usepackage{authblk}
\usepackage{caption}
\usepackage{fancyhdr}
\usepackage[hidelinks]{hyperref}

\hypersetup{
  pdfauthor={Tim Steinert and David Ginsbourger},
  pdftitle={Triply-Scalable Equivariant Gaussian Process Modeling},
  pdfsubject={Gaussian processes, equivariance, and scalable inference},
  pdfkeywords={equivariance, sparse Gaussian processes, matrix-valued kernels, uncertainty quantification, molecular machine learning}
}
\AtBeginDocument{%
  }

\providecommand{\Description}[1]{}

\newcommand{\PiA}{\boldsymbol{\Pi}_s}
\newcommand{\Kpi}{K_{\boldsymbol{\Pi}}}
\newcommand{\sectionmap}{s}
\newcommand{\KA}{K_{\overline{A}}}
\newcommand{\R}{\mathbb{R}}

\newcommand{\Kbase}{K_o}
\newcommand{\bff}{\boldsymbol{f}}
\newcommand{\bmm}{\boldsymbol{m}}
\newcommand{\bmu}{\boldsymbol{\mu}}

\newcommand{\bx}{\boldsymbol{x}}
\newcommand{\bxast}{\boldsymbol{x^\ast}}

\newcommand{\by}{\boldsymbol{y}}
\newcommand{\bz}{\boldsymbol{z}}
\newcommand{\ba}{\boldsymbol{a}}
\newcommand{\bZ}{\boldsymbol{Z}}
\newcommand{\be}{\boldsymbol{e}}
\newcommand{\bv}{\boldsymbol{v}}

\newcommand{\btheta}{\boldsymbol{\theta}}
\newcommand{\bphi}{\boldsymbol{\phi}}
\newtheorem{thm}{Theorem}[section]
\newtheorem{cor}[thm]{Corollary}

\newtheorem{lem}[thm]{Lemma}

\newtheorem{rem}[thm]{Remark}

\begin{document}
\title{Triply-Scalable Equivariant Gaussian Process Modeling}

\author{\href{https://orcid.org/0009-0005-1290-2313}{Tim Steinert}$^{\ast}$}
\author{\href{https://orcid.org/0000-0003-2724-2678}{David Ginsbourger}}
\affil{Institute of Mathematical Statistics and Actuarial Science, University of Bern, Bern, Switzerland\\
$^{\ast}$Corresponding author: \href{mailto:tim.steinert@unibe.ch}{\texttt{tim.steinert@unibe.ch}}}
\date{}

\renewcommand{\Authfont}{\large}
\renewcommand{\Affilfont}{\footnotesize}
\setlength{\affilsep}{0.5em}

\maketitle
\thispagestyle{plain}
\enlargethispage{3\baselineskip}
\vspace{-3.5em}

\begin{figure}[!htbp]
  \centering
  \includegraphics[width=\textwidth]{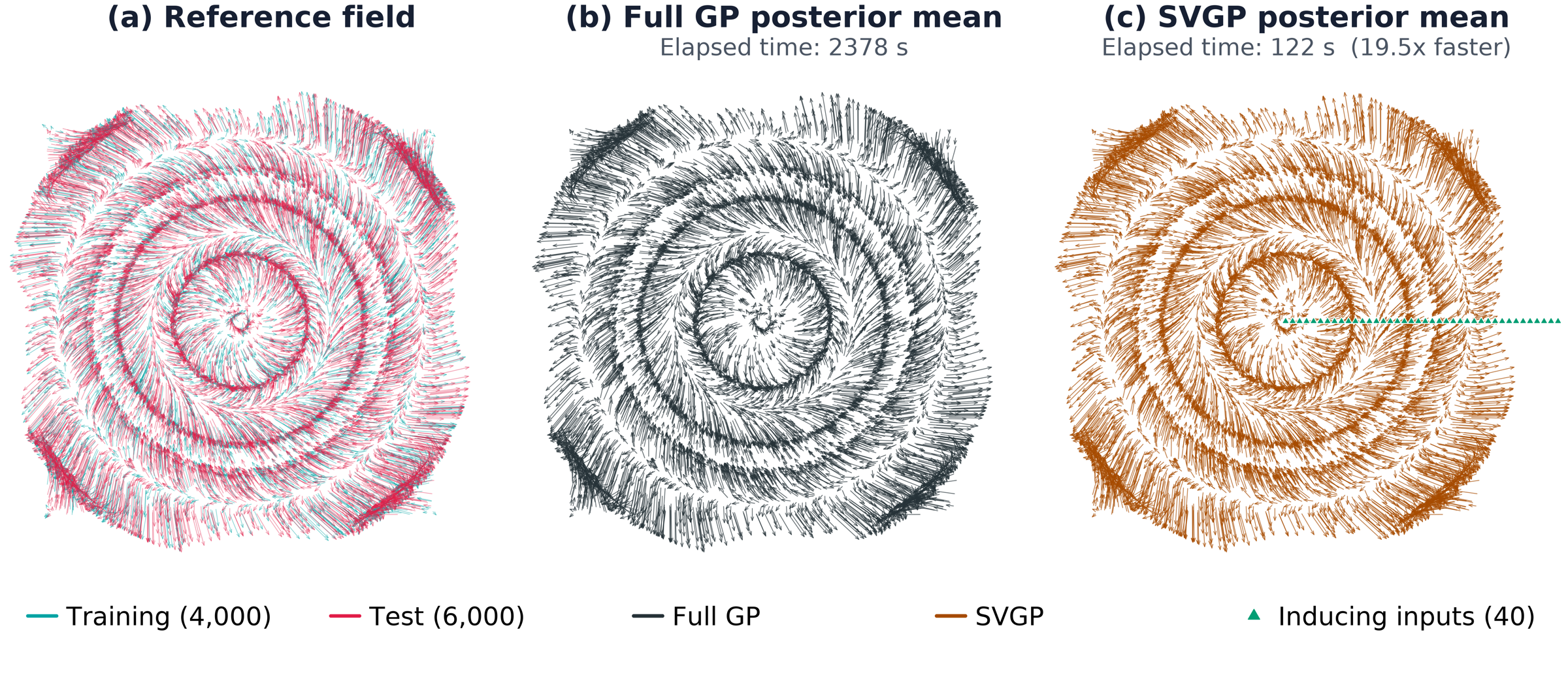}
  \caption{$\mathrm{SO}(2)$-equivariant GP regression with the integration-free equivariant kernel $K_{\Pi}$; see Experiment~\ref{subsubsec:so2-exp1}. Left: a stochastically equivariant GP realization serving as ground truth, split into 40\% training and 60\% test data. Middle and right: posterior means from a full GP and a sparse variational GP (SVGP). The SVGP uses inducing points in the fundamental region and achieves substantial speedups while preserving equivariance.}
  \Description{Three panels compare an SO(2)-equivariant vector field with its posterior reconstructions. The left panel overlays training and held-out vectors from the reference field. The middle panel shows the full-GP posterior mean. The right panel shows the SVGP posterior mean and marks the inducing inputs.}
  \label{fig:teaser}
\end{figure}

\begin{abstract}
Gaussian processes (GPs) provide principled probabilistic predictions while encoding prior knowledge, including equivariances. Yet, their use in large-scale scientific problems is limited by computational cost. Equivariant neural networks are common but typically lack the uncertainty quantification offered by GPs, which is valuable in applications such as molecular research. High-dimensional inputs and large symmetry groups further demand scalability.  We establish results pertaining to the interplay of GP equivariance and conditioning and leverage them to obtain equivariant sparse GPs through suitable mean functions and covariance kernels. We instantiate this framework with a flexible class of integration-free equivariant kernels, yielding scalable and data-efficient GP inference. In particular, we introduce triply scalable equivariant Gaussian processes.

We employ equivariant sparse variational Gaussian processes for $\mathrm{SO}(2)$-equivariant vector fields and molecular property prediction. Alongside the SVGP, we develop a matrix-free equivariant full-GP implementation that combines an exact Kronecker reduction with preconditioned conjugate-gradient solves, enabling fast and scalable evaluation of the full joint predictive density. We further compare different approaches for selecting inducing points in the equivariant sparse GP models. Our test cases include synthetic $\mathrm{SO}(2)$-equivariant fields as well as the prediction of electric dipole moments of N-methylformamide based on quantum chemistry simulations, achieving accurate, uncertainty-aware predictions at a fraction of the computational cost of classical GP inference.
\end{abstract}

\medskip
\newpage
\section{Introduction}
\label{sec:introduction}
The incorporation of structural knowledge such as physical laws into machine learning models has gained significant attention as a means of improving predictive accuracy while maintaining physical realism. Among the most important forms of structure are equivariances, where transformations of the inputs induce predictable transformations of the outputs. In molecular quantum chemistry, for instance, applying a simultaneous rigid motion to all atom positions rotates vector-valued properties such as electric dipole moments accordingly. The value of such equivariances is well established in deep learning, where group-equivariant architectures exploit information across entire orbits of a group action from a single data point and thereby achieve powerful large-scale symmetry-aware models \citep{cohen2016group,weiler20183d,thomas2018tensor,fuchs2020se,pmlr-v139-schutt21a,batzner20223,kovacs2023evaluation,satorras2021n}. Despite their empirical success, however, such models typically do not yield probabilistically closed posteriors with exact equivariance guarantees, which limits their direct use in settings where calibrated uncertainty quantification is required.

Kernel methods provide a complementary perspective by encoding equivariance directly in the covariance structure of a stochastic process. Classical constructions achieve this through group averaging or symmetry-adapted representations \citep{PhysRevB.95.214302,holderrieth2021equivariant,Soapgp}, while derivative-based approaches model vector fields as gradients of invariant scalar potentials \cite{SS2012,berlinghieri2023gaussian,Sun_2022}. The latter is highly effective when a potential-based description is appropriate, but it imposes a specific structural assumption that does not apply to all equivariant vector-valued targets. Within Gaussian process (GP) modelling, incorporating equivariances is particularly appealing because GPs provide an explicit probabilistic framework that naturally supports uncertainty quantification, active learning, and decision-making under uncertainty \citep{moss2020gaussian,griffiths2022data,rankovic2024bayesian}. At the same time, it is technically demanding: enforcing equivariance for vector-valued targets requires matrix-valued covariance kernels that satisfy positive definiteness under structural constraints, which significantly restricts admissible constructions. This has led to a growing literature on structured GPs, including curl- or divergence-free models \cite{SS2012} and invariant or equivariant kernels derived from physical symmetries \cite{AFST_2012_6_21_3_501_0,GINSBOURGER2016117,van2018learning,holderrieth2021equivariant,henderson2023pde}. A persistent obstacle in this line of work is computational. Many principled equivariant kernels rely on explicit group integration or Haar averaging \citep{JMLR:v8:reisert07a}, or on representation-theoretic basis expansions \citep{Soapgp}, which are mathematically elegant but often computationally demanding. This burden becomes particularly pronounced for vector-valued outputs and molecular systems, where kernel evaluation itself can dominate the overall cost, leading to a trade-off between expressiveness and efficiency. Recent work \citep{steinertintegration} addressed this issue by characterizing stochastic equivariance for centred second-order vector-valued random fields and introducing {integration-free} equivariant kernel matrix functions constructed via projections onto fundamental regions \citep{aslan23a,AFST_2012_6_21_3_501_0}. These constructions avoid explicit group integration while retaining the flexibility of a chosen base kernel, thereby enabling lightweight and stable equivariant GPs for vector fields and molecular dipole moments.

Fast equivariant kernels alone, however, do not yet make equivariant GP modelling scientifically useful at scale. For practical probabilistic modelling, equivariance must hold at the level of full posterior distributions and remain valid under the scalable GP approximations required by modern datasets. Understanding whether and how equivariance interacts with conditioning and approximation is therefore essential for connecting equivariant kernel design with large-scale inference. Our main contributions revolve around the interplay between stochastic equivariance of GPs and of their conditional distributions. In particular, broad classes of sparse and approximate GP constructions inherit stochastic equivariance both a priori and a posteriori, which in turn bridges the gap between integration-free equivariant kernels and scalable probabilistic GP modelling. We instantiate this using sparse variational Gaussian processes and PCG-based matrix-free full-GP inference, and demonstrate the resulting models on $\mathrm{SO}(2)$-equivariant examples and on emulating dipole moments of N-methylformamide resulting from quantum chemistry computations.

\paragraph{Contributions}
The main contributions of this work are threefold:
\begin{enumerate}
    \item We prove that stochastic equivariance is preserved under GP conditioning (Theorem~\ref{thm:equivariant-conditioning}), so conditional GPs inherit stochastic equivariance whenever the unconditional mean function and matrix-valued covariance kernel are equivariant. We also prove in Theorem~\ref{thm:finite response equiv} and the subsequent corollary that conditional equivariance for a finite number of realized responses is sufficient to conclude that equivariance holds unconditionally.
    \item We show that common scalable (sparse) GP constructions, including inducing-point and related approximations, preserve stochastic equivariance when built on such kernels.
    \item Building on integration-free equivariant kernel matrix functions \cite{steinertintegration}, we provide large-scale stochastically equivariant GP models via equivariant sparse variational inference and PCG-based matrix-free full-GP inference, demonstrated on $\mathrm{SO}(2)$-equivariant fields and molecular dipole moment prediction.
\end{enumerate}

\paragraph{Outline.} Section~\ref{subsec:background} reviews vector-valued GPs, followed by an introduction to stochastic equivariance in Section~\ref{subsec:inv-eq-background}. Section~\ref{sec:sparse-gps} reviews scalable inference paradigms. Section~\ref{sec:eq_gps} then establishes propagation of stochastic equivariance under conditioning and its consequences for scalable GPs. Section~\ref{sec:triply-scalable} presents the numerical experiments, followed by the conclusion in Section~\ref{sec:discussion}.

\subsection{Background}
\label{subsec:background}

\newcommand{\myf}{\boldsymbol{f}}

\subsubsection{Basics of vector-valued Gaussian processes}
\label{subsec:vv-gps}
All random quantities are defined on a common probability space $(\Omega,\mathcal{A},\mathbb{P})$. We consider mappings from 
a set $D$ to $\mathbb{R}^p$, and while our constructs and results hold in more generality, we often assume for simplicity that $D \subset \mathbb{R}^d$. Throughout the paper, vector-valued quantities are written in boldface and matrix-valued quantities in uppercase (In the scalar-output review of Section~\ref{sec:sparse-gps}, we retain standard scalar GP notation).
For $\myf(\boldsymbol{x})\in\mathbb{R}^p$, we write
\begin{displaymath}\myf(\boldsymbol{x}) = \bigl( f^{(1)}(\boldsymbol{x}),\dots,f^{(p)}(\boldsymbol{x}) \bigr),\end{displaymath}
and use superscripts in parentheses for output coordinates, reserving subscripts $i,j$ for input locations.
While the notation 
$\bigl(\myf(\boldsymbol{x})\bigr)_{\boldsymbol{x}\in D}$ 
is used for $\mathbb{R}^p$-valued Gaussian processes indexed by $D$, we liberally use the notation $\myf(\boldsymbol{x})$ to also speak of the unknown fixed mapping to be approximated, as often done in the machine-learning literature (from a Bayesian perspective, $\myf$ is a fixed but unknown mapping with a GP prior). 

A vector-valued GP is characterized in distribution by its mean function (throughout assumed zero unless specified otherwise) and a matrix-valued covariance kernel
$K : D \times D \to \mathbb{R}^{p\times p}$, defined componentwise by
\begin{displaymath}\bigl[K(\boldsymbol{x},\boldsymbol{x}')\bigr]_{rs} = \operatorname{Cov}\!\left(f^{(r)}(\boldsymbol{x}), f^{(s)}(\boldsymbol{x}')\right).\end{displaymath}
For any finite collection of inputs in $D,$ $X =(\boldsymbol{x}_1, \ldots,\boldsymbol{x}_n)$, we define the stacked vector \begin{displaymath}\myf(X) := \bigl( \myf(\boldsymbol{x}_1);\dots; \myf(\boldsymbol{x}_n) \bigr) \in \mathbb{R}^{pn}.\end{displaymath}

\noindent
Under the centred GP assumption on $\myf$, $\myf(X)$ is a multivariate Gaussian vector,
\begin{displaymath}\myf(X) \sim \mathcal{N}\bigl(\bm 0,\, K(X,X)\bigr),\end{displaymath}
where $K(X,X)\in\mathbb{R}^{pn\times pn}$ is made of $n\times n$ blocks of dimension $p\times p$ with $(i,j)$-block $K(\boldsymbol{x}_i,\boldsymbol{x}_j)$. Observed data are modeled by $\by$ assumed to be a realization from
\begin{displaymath}\boldsymbol{Y}=\myf(X)+\boldsymbol{\varepsilon},\end{displaymath}
where $\boldsymbol{\varepsilon}\sim\mathcal{N}(\bm 0,\sigma^2 I_{pn})$ is independent of $\myf$. For test inputs $\boldsymbol{x}_\ast, \boldsymbol{x}_\ast'$, the posterior GP is obtained by Gaussian conditioning. Throughout the paper, we use the shorthand notation
\begin{equation*}
    (\,\cdot\,\mid \boldsymbol{Y}=\by)\equiv(\,\cdot\,\mid \by),
\end{equation*}
that is, probabilistic quantities conditional on the observation events such as $\boldsymbol{Y}=\by$ are written simply as conditional on the realization $\by$. Writing $K:=K(X,X),$ the posterior mean and covariance are given by
\begin{align*}
\mathbb{E}[\myf(\boldsymbol{x}_\ast)\mid \by]
&:=
\mathbb{E}[\myf(\boldsymbol{x}_\ast)\mid \myf(X)+\boldsymbol{\varepsilon}=\by]
=
K(\boldsymbol{x}_\ast,X)\,(K+\sigma^2 I)^{-1}\by,\\
\operatorname{Cov}\bigl(\myf(\boldsymbol{x}_\ast),\myf(\boldsymbol{x}_\ast')\mid \by\bigr)
&:=\operatorname{Cov}\bigl(
\myf(\boldsymbol{x}_\ast),\myf(\boldsymbol{x}_\ast')
\mid \myf(X)+\boldsymbol{\varepsilon}=\by
\bigr)\\
&=
K(\boldsymbol{x}_\ast,\boldsymbol{x}_\ast')
-
K(\boldsymbol{x}_\ast,X)\,(K+\sigma^2 I)^{-1}K(X,\boldsymbol{x}_\ast').
\end{align*}
Kernel and likelihood hyperparameters may be selected by maximizing the marginal likelihood, by MAP estimation under priors on the hyperparameters, by predictive criteria such as cross-validation, or, in sparse variational settings, by maximizing an evidence lower bound \citep[Chapter~5]{GPsforMLbook}. See also \citep{Titsias2009VariationalLO,Hensman2013GaussianPF} for the variational case.
In the experiments below we focus on likelihood- and ELBO-based estimation.
For a fixed prior GP, marginal likelihood estimation takes the form
\begin{equation}\label{eq:likelihood} l(\btheta)=\log p(\by\mid \btheta) = -\tfrac12 \by^\top (K+\sigma^2 I)^{-1}\by -\tfrac12 \log\det(K+\sigma^2 I) -\tfrac{pn}{2}\log(2\pi).\end{equation}

\begin{table}[t]
\centering
\small
\caption{Main notation used throughout the paper.}
\begin{tabular}{@{}ll@{}}
\toprule
Symbol & Meaning \\
\midrule
$D \subset \mathbb{R}^d$ & input domain \\
$\boldsymbol{x},\boldsymbol{x}'$ & input locations in $D$ \\
$\myf=(\myf(\boldsymbol{x}))_{\boldsymbol{x}\in D}$ & $\mathbb{R}^p$-valued Gaussian process \\
$\myf(\boldsymbol{x}) \in \mathbb{R}^p$ &  (GP) response vector at input $\boldsymbol{x}$ \\
$f^{(r)}(\boldsymbol{x})$ & $r$th component of $\myf(\boldsymbol{x})$ \\
$K(\boldsymbol{x},\boldsymbol{x}') \in \mathbb{R}^{p\times p}$ & matrix-valued covariance kernel \\
$[K(\boldsymbol{x},\boldsymbol{x}')]_{ij}$ & $(i,j)$ entry of the kernel matrix \\
$X=(\boldsymbol{x}_1,\dots,\boldsymbol{x}_n)$ & training inputs \\
$\myf(X)$ & stacked training vector \\
$\boldsymbol{y}_i \in \mathbb{R}^p$ & observation at input $\boldsymbol{x}_i$ \\
$\by$ & stacked observed training vector \\
$\btheta$ & hyperparameter vector \\
$K=K(X,X)$ & block Gram matrix on the training inputs\\
\bottomrule
\end{tabular}
\end{table}

In the vector-valued equivariant setting, the central modelling object is therefore a matrix-valued covariance kernel subject to symmetry constraints. The next subsection makes these constraints explicit and studies the corresponding notions of invariance and equivariance.

\subsection{Invariances and equivariances in GP  models}
\label{subsec:inv-eq-background}

To describe invariances and equivariances, we follow the  equivariant kernel-based modelling setup from \cite{JMLR:v8:reisert07a,steinertintegration}.
Let $(G,\circ)$ be a 
group acting on $D$ via a 
left group action $\star : G \times D \to D$ (thus satisfying $e \star \bx = \bx$ and $g \star (h \star \bx) = (g \circ h) \star \bx$). We further assume $G$ to be a linear group with a finite-dimensional representation  $\rho : G \to \mathbb{R}^{p \times p}$ (thus satisfying $\rho_{g \circ h} = \rho_g \rho_h$). In this context, a deterministic mapping $\boldsymbol{\varphi} : D \to \mathbb{R}^p$ is said equivariant if for any $\bx \in D,$ $g\in G$ holds  \begin{math}
    \boldsymbol{\varphi}(g \star \bx)=\rho_g \boldsymbol{\varphi}(\bx).
\end{math}
Let us note that this property can be cast as being in the null space of combination of matrix-weighted composition operators.
Indeed, revisiting the composition-operator viewpoint of \cite{GINSBOURGER2016117} in the present vector-valued setting, any mapping $\bm v : D \to D$ induces a composition operator \begin{displaymath}T_{\bm v} : \boldsymbol{\varphi}\mapsto \boldsymbol{\varphi}\circ \bm v.\end{displaymath}
More generally, for mappings $\bm v_1,\dots,\bm v_q : D \to D$ and matrix weights $A_1,\dots,A_q \in \mathbb{R}^{p\times p}$, we call
\begin{displaymath}T=\sum_{i=1}^q A_i T_{\bm v_i}\end{displaymath}
a combination of matrix-weighted composition operators. 
Invariance is obtained as the special case of the trivial representation, i.e. when $\rho_g=I_p$ for all $g\in G$, in which case equivariance reduces to $\boldsymbol{\varphi}(g\star\bx)=\boldsymbol{\varphi}(\bx)$.
For fixed $g\in G$, letting $\bm v_g(\bx)=g\star\bx$, equivariance can be rewritten as invariance under the particular operator
\begin{displaymath}T_g = T_{\bm v_g}-\rho_g T_{\mathrm{id}}, \qquad\text{that is}\qquad (T_g\boldsymbol{\varphi})(\bx) := \boldsymbol{\varphi}(g\star\bx)-\rho_g\boldsymbol{\varphi}(\bx),\end{displaymath}
so that $\boldsymbol{\varphi}$ is equivariant if and only if $T_g\boldsymbol{\varphi}=0$ for all $g\in G$.
Thus equivariance appears as a degeneracy condition under a combination of matrix-weighted composition operators, while invariance corresponds to the same construction with $\rho_g=I_p$.
Analogously, when we consider an $\mathbb{R}^p$-valued random field $\bZ = (\bZ_{\bx})_{\bx \in D},$ stochastic equivariance is defined by
\begin{equation*}
\forall \bx \in D, \ g \in G, \qquad \mathbb{P}\left(T_g\left(\bZ\right)_{\bx}=0\right)=1.\end{equation*} As proven in \citep{steinertintegration}, in the second-order, zero mean case this is equivalent to \begin{equation}\label{eq:kernel-equivariance} \forall \bx,\bx' \in D,\ \forall g,h \in G, \qquad K(g \star \bx, h \star \bx') = \rho_g K(\bx,\bx') \rho_h^\top .\end{equation}
In the same operator language, for each fixed $\bx'\in D$ the kernel feature map $\bx\mapsto K(\bx,\bx')$ is $\mathbb{R}^{p\times p}$-valued, and equivariance in the first argument can be written as
\begin{displaymath}(T_g(K(\cdot,\bx')))(\bx) := K(g\star\bx,\bx')-\rho_gK(\bx,\bx'),\end{displaymath}
so that $T_g(K(\cdot,\bx'))=0$ for all $g\in G$ and $\bx'\in D$.
Applying the analogous condition in the second argument yields the full kernel identity \eqref{eq:kernel-equivariance}.
This condition is standard in kernel-based equivariant learning \citep{JMLR:v8:reisert07a} and is stronger than the diagonal constraint $K(g \star \bx, g \star \bx') = \rho_g K(\bx,\bx') \rho_g^\top$ commonly used in the literature (see e.g. \citep{holderrieth2021equivariant}), which provides a weaker sufficient condition of equivariant posterior means.

Invariant kernels correspond to the same condition under the trivial representation and therefore satisfy
\begin{displaymath}K(g\star \bx, h\star \bx') = K(\bx,\bx') \qquad \forall g,h\in G.\end{displaymath}
In practice, invariance is typically enforced either by constructing invariant descriptors or distances, or by symmetrizing a base kernel with respect to the group action.
Invariant kernel constructions and their use in Gaussian process modelling are discussed, for example, in \cite{AFST_2012_6_21_3_501_0} and \cite{van2018learning}.

\subsubsection{Reynolds operators for invariance and equivariance}\label{subsec:Reynolds}
For finite groups, invariance and equivariance can be enforced by explicit group averaging.
Let $G$ be a finite group acting on $D$ via $\star$, and let $\rho : G \to \mathbb{R}^{p\times p}$ be a 
representation. 
Given a matrix-valued base kernel $\Kbase$, the equivariant Reynolds operator produces a new kernel
\begin{equation}\label{eq:reynolds-equivariant} K^{R}(\bx,\bx') = \frac{1}{|G|^2} \sum_{g,h\in G} \rho_g^{\top}\, \Kbase(g\star \bx, h\star \bx')\, \rho_h .\end{equation}
The kernel $K^{R}$ is equivariant with respect to $\star$ and $\rho$, and remains positive definite whenever $\Kbase$ is.
In the special case where $\rho$ is the trivial representation, Eq.~\eqref{eq:reynolds-equivariant} reduces to the standard
Reynolds operator enforcing invariance. Reynolds operators provide a discrete analogue of Haar integration and are particularly useful for permutation, reflection, or discrete rotational symmetries.

\subsubsection{Equivariant matrix-valued kernels via group integration}
Following the general framework of \cite{JMLR:v8:reisert07a}, assume further that $G$ is compact and unimodular and that $\rho : G \to \mathbb{R}^{p\times p}$ is a continuous unitary representation. Under these assumptions, equivariant matrix-valued kernels can be constructed by Haar integration.

Given a base matrix-valued kernel $\Kbase : D\times D \to \mathbb{R}^{p\times p},$ 
one obtains an equivariant kernel satisfying \eqref{eq:kernel-equivariance} by
\begin{equation}\label{eq:double-integral-kernel} K_{\int}(\boldsymbol{x},\boldsymbol{x}') = \int_{G^2} \rho_g^{\top}\, \Kbase(g\star \boldsymbol{x}, h\star \boldsymbol{x'}) \,\rho_h \;\mathrm{d}g\,\mathrm{d}h,\end{equation}
where $\mathrm{d}g$ denotes the normalized Haar measure on $G$.
Equivariance of $K_{\int}$ follows from the translation invariance of the Haar measure and unimodularity of the representation.
If $\Kbase$ already satisfies the equivariance condition
$ \Kbase(g\star \bx, h\star \bx') = \rho_g \Kbase(\bx,\bx') \rho_h^\top$,
normalization implies $K_{\int}=\Kbase$.
While this approach provides a canonical construction of equivariant kernels, the double integration over $G^2$ can make
equivariant Gaussian process modelling computationally demanding. By giving up expressiveness of the base kernel by taking a scalar-valued $k_o$ which is invariant, one can reduce the double integration in \eqref{eq:double-integral-kernel} to a single, more tractable integral to obtain
\begin{displaymath}K_{\int}(\boldsymbol{x},\boldsymbol{x}') = \int_{G} k_o(\boldsymbol{x}, g\star \boldsymbol{x'})\rho_g \;\mathrm{d}g,\end{displaymath}
which has also been used in \cite{PhysRevB.95.214302}.


\subsubsection{Integration-free equivariant kernels via fundamental regions}
Exact GP inference scales as $\mathcal{O}(n^3)$ time and $\mathcal{O}(n^2)$ memory. There exist well established approaches to scalable GP approximations which we will recall in Section \ref{sec:sparse-gps}. Scalable inference in the vector-valued equivariant case is however meaningless when the computational complexity is dominated by cost of evaluating the kernel. If we wish to have a flexible equivariant kernel, we therefore must resort to an integration-free approach that allows for general base kernels but circumvents the need for numerical quadrature that would be required to evaluate \eqref{eq:double-integral-kernel}. An alternative approach to explicit group integration, proposed in \citep{steinertintegration}, exploits the geometry of group actions. A subset $A\subset D$ is called a fundamental region of the action if $G\star \overline{A} = D$ and $(g\star A)\cap A=\emptyset$ for all $g\in G\setminus\{e\}$. For such an $A$ and any $\boldsymbol{x}\in D$, there exists at least one $g\in G$ such that $g\star\boldsymbol{x}\in\overline{A},$ with $g$ being unique whenever $\boldsymbol{x}$ lies in the interior $G\star A.$

A section is any mapping $s : D\to G$ satisfying
\begin{displaymath}s(\boldsymbol{x})\star\boldsymbol{x}\in\overline{A},\end{displaymath}
and we denote by $\PiA : D\to\overline{A}$ the associated projection.
Given any matrix-valued kernel $\KA$ on $\overline{A}\times\overline{A}$, the kernel
\begin{equation}\label{eq:fundamental region kernel} \Kpi(\boldsymbol{x}, \boldsymbol{x}') = \rho_{\sectionmap(\boldsymbol{x})}^{\top} \KA\bigl(\PiA(\boldsymbol{x}), \PiA(\boldsymbol{x}')\bigr) \rho_{\sectionmap(\boldsymbol{x'})}\end{equation}
defines an equivariant matrix-valued kernel on $(G\star A)\times(G\star A)$.
This construction avoids explicit group integration while retaining the full flexibility of the base kernel. Stochastic equivariance holds globally for free group actions and holds up to boundary effects otherwise.

\subsubsection{Gradient-based equivariance and derivative kernels.}
A further mechanism for equivariance exploits additional physical structure of the target quantity.
For molecular dipole moments, \cite{Sun_2022} model the dipole as the gradient of an invariant scalar potential.
In Gaussian process models, this leads to covariance kernels obtained as mixed derivatives of a base kernel,
schematically of the form
\begin{displaymath}K_{\mu}(x,x') = \nabla_x \nabla_{x'} k(x,x').\end{displaymath}
Such derivative kernels yield equivariance by construction but rely on the assumption that the vector-valued target arises as a gradient field.

\section{Sparse and approximate GPs}
\label{sec:sparse-gps}

The training cost of $\mathcal{O}(n^3)$ time and $\mathcal{O}(n^2)$ memory required for GP regression comes from solving linear systems with $K+\sigma^2 I.$ When $n$ is large, one therefore resorts either 
\begin{itemize}
\item to approximations that reduce the effective rank of $K$, 
\item or to introduce inducing variables and optimize a tractable variational objective, 
\item or to replace the kernel by structured or feature-based surrogates, 
\item or to keep the exact kernel but approximate the linear solves via iterative methods. 
\end{itemize}

For clarity we present the scalar-output case. The vector-valued case follows by block-matrix notation as in Section~\ref{subsec:vv-gps}. Although the  GP is scalar-valued in this section, finite-dimensional observation, inducing, and weight vectors are written in bold to remain consistent with the global notation. We consider the scalar-valued regression model $f\sim\mathrm{GP}(0,k)$ with $y_i$ realization of $f(x_i)+\varepsilon_i$ where $\varepsilon_i\sim\mathcal{N}(0,\sigma^2)$. Let $\by:=(y_1,\dots,y_n)$, $\boldsymbol{f}:=(f(x_1),\dots,f(x_n))$, $X=\left(x_1,\ldots,x_n\right)$, $K=K(X,X)$, and for a test point $x_\ast$ write $\boldsymbol{k}_\ast=k(x_\ast,X)$ and $k_{\ast\ast}=k(x_\ast,x_\ast)$.
The exact posterior predictive distribution is
\begin{displaymath}p(f_\ast\mid \by)=\mathcal{N}\!\big(\boldsymbol{k}_\ast(K+\sigma^2I)^{-1}\by,\;k_{\ast\ast}-\boldsymbol{k}_\ast(K+\sigma^2I)^{-1}\boldsymbol{k}_\ast^\top\big).\end{displaymath}
\subsection{Subset-of-data and Nyström-type approximations}
Early approaches reduce computation by restricting attention to a subset of the training data. In the {subset-of-data} (SoD) approach, one selects $X_s\subset X$ with $|X_s|=m\ll n$ and forms the exact GP posterior using only $(X_s,\by_s)$, where $\by_s$ denotes the subvector of $\by$ corresponding to $X_s$. The predictive distribution is then the exact GP formula with $K(X,X)$ replaced by $K(X_s,X_s)$. A related low-rank approximation is the Nyström method \citep{Williams2000UsingTN}. Given an inducing subset $U\subset X$ with $|U|=m$, one approximates $K(X,X)$ by $\tilde K(X,X)=K(X,U)K(U,U)^{-1}K(U,X)$. This yields a reduced-rank GP obtained by replacing $K$ with $\tilde K$ in the exact posterior formulas.

\subsection{Inducing-point prior approximations}
A major step was to introduce the inducing vector $\boldsymbol{u}:=f(Z)=(f(z_1),\dots,f(z_m))$ at pseudo-inputs $Z=\left(z_1,\dots,z_m\right)$ and the low-rank matrix \begin{math}
    Q(a,b)=K(a,Z)K(Z,Z)^{-1}K(Z,b).
\end{math}
The {subset-of-regressors} (SoR) approximation
\citep{silverman1985some,smola2000sparse,Candela2005AUV} replaces $K(X,X)$ by $Q(X,X)$.
The {deterministic training conditional} (DTC) \citep{csato2002sparse,pmlr-vR4-seeger03a,Candela2005AUV} uses $Q(X,X)$ in the data-fit term but retains $k_{\ast\ast}$ at test points.
To correct the overly confident training variances of these low-rank priors, the
{fully independent training conditional} (FITC) \citep{Snelson2005SparseGP,pmlr-v2-snelson07a,Candela2005AUV} introduces a diagonal correction $\Lambda=\mathrm{diag}(K(X,X)-Q(X,X))+\sigma^2I$.
The related PITC approximation replaces the diagonal correction by a block-diagonal one.

\subsection{Variational inducing-point methods}
Variational inducing-point methods differ fundamentally from earlier sparse GP constructions in that they do not approximate the prior covariance directly. Instead, they approximate the {posterior distribution} of the  function by restricting it to a tractable family of Gaussian processes. The central idea is to introduce a small set of inducing variables that summarize the information in the data and to choose the best approximate posterior, within this restricted family, by minimizing a Kullback--Leibler divergence to the exact posterior.

\subsubsection{Variational free energy (VFE)}
The variational free energy framework of \cite{Titsias2009VariationalLO} provides a principled inducing-point approximation whose objective minimizes the KL divergence to the exact GP posterior and becomes exact when the inducing variables provide a complete representation of the training covariance. 
Using the inducing vector $\boldsymbol{u}=f(Z)$ introduced above, consider a Gaussian variational posterior $q(\boldsymbol{u})=\mathcal{N}(\bmu,A)$ approximating $p(\boldsymbol{u}\mid \by)$.
Crucially, the conditional GP $p(f\mid \boldsymbol{u})$ induced by the prior is kept exact.
The approximate posterior process is then defined as
\begin{displaymath}q(f)=\int p(f\mid \boldsymbol{u})q(\boldsymbol{u})\,d\boldsymbol{u},\end{displaymath}
which is itself a Gaussian process. This equality holds by assuming $\boldsymbol{u}$ is sufficient, i.e. \begin{math}
    p(f\mid \boldsymbol{u},\by)=p(f\mid \boldsymbol{u}).
\end{math}
Its mean and covariance are given by
$m^q(x)=\boldsymbol{k}(x,Z)K(Z,Z)^{-1}\bmu$ and
\begin{displaymath}
k^q(x,x')=k(x,x')-\boldsymbol{k}(x,Z)K(Z,Z)^{-1}
\bigl(K(Z,Z)-A\bigr)K(Z,Z)^{-1}\boldsymbol{k}(Z,x').
\end{displaymath}
The variational parameters $(\bmu,A)$ and hyperparameter vector $\btheta$ are learned by maximizing the evidence lower bound (ELBO)
\begin{displaymath}\mathcal{L}(\bmu,A,\btheta) = \mathbb{E}_{q(f)}[\log p(\by\mid \boldsymbol{f})] - \mathrm{KL}\!\left[q(\boldsymbol{u})\,\|\,p_{\btheta}(\boldsymbol{u})\right], \qquad p_{\btheta}(\boldsymbol{u})=\mathcal{N}(\bm 0,K(Z,Z)).\end{displaymath}
For a Gaussian likelihood $p(\by\mid \boldsymbol{f})=\mathcal{N}(\by;\boldsymbol{f},\sigma^2I)$, the ELBO admits a closed-form expression and the optimal
variational parameters $(\bmu^\star,A^\star)$ can be computed analytically.
Writing $\Sigma=(K(Z,Z)+\sigma^{-2}K(Z,X)K(X,Z))^{-1}$, one obtains
\begin{equation*}\bmu^\star=\sigma^{-2}K(Z,Z)\Sigma K(Z,X)\by \qquad \text{and} \qquad A^\star=K(Z,Z)\Sigma K(Z,Z).\end{equation*}
Substituting these expressions yields a collapsed variational objective depending only on inducing locations $Z$ and
hyperparameters $\btheta$:
\begin{equation}
\begin{aligned}
\mathcal{L}^\star(\btheta,Z)
&=\log\mathcal{N}\!\bigl(\by;\bm 0,\;Q(X,X)+\sigma^2 I\bigr)
-\frac{1}{2\sigma^2}\operatorname{tr}\!\bigl(K(X,X)-Q(X,X)\bigr),\\
Q(X,X)&=K(X,Z)K(Z,Z)^{-1}K(Z,X).
\end{aligned}
\label{eq:SVGP elbo}
\end{equation}
The trace term penalizes inducing-point configurations that fail to capture the full covariance structure, preventing the
approximate posterior from becoming overconfident away from the inducing locations.
This collapsed bound can be optimized with respect to $(\btheta,Z)$ using gradient-based methods at cost $\mathcal{O}(nm^2)$.

\subsubsection{Stochastic variational GP (SVGP)}
The stochastic variational GP of \cite{Hensman2013GaussianPF} extends the VFE framework to large datasets by enabling stochastic optimization. The same variational family $q(\boldsymbol{u})=\mathcal{N}(\bmm,S)$ is retained, meaning we do not plug in the ELBO-optimal $\bmm^\star,\ S^\star.$ The ELBO is rewritten in a form that decomposes over data points, allowing unbiased minibatch estimates of its gradient.
The approximate posterior process has predictive mean and covariance
\begin{align*}
\mathbb{E}_q[f(x_\ast)]
&=\boldsymbol{k}(x_\ast,Z)K(Z,Z)^{-1}\bmm,\\
\operatorname{Cov}_q[f(x_\ast)]
&=k_{\ast\ast}-\boldsymbol{k}(x_\ast,Z)K(Z,Z)^{-1}
\bigl(K(Z,Z)-S\bigr)K(Z,Z)^{-1}\boldsymbol{k}(Z,x_\ast).
\end{align*}
For Gaussian likelihoods, the ELBO decomposes as
\begin{equation}\label{eq:ELBO_svgp_sum} \mathcal{L}(\bmm,S,\btheta)=\sum_{i=1}^n \ell_i(\bmm,S,\btheta)-\mathrm{KL}\!\left[q(\boldsymbol{u})\,\|\,p_{\btheta}(\boldsymbol{u})\right], \qquad \ell_i = \mathbb{E}_{q(f_i)}[\log p(y_i\mid f_i)].\end{equation}
For the Gaussian likelihood $p(y_i\mid f_i)=\mathcal{N}(y_i;f_i,\sigma^2)$, each term admits the closed form \begin{displaymath}\ell_i=-\tfrac12\Bigl(\log(2\pi\sigma^2)+\sigma^{-2}\bigl[(y_i-\mu_i)^2 + \sigma_i^2\bigr]\Bigr),\end{displaymath}
where $\mu_i = \mathbb{E}_q[f_i]$ and $\sigma_i^2 = \mathrm{Var}_q(f_i)$.

Let $\mathcal{B}\subset\{1,\dots,n\}$ be a minibatch of size $B=|\mathcal{B}|$ sampled uniformly (with or without replacement). An unbiased estimator of the ELBO is
\begin{equation}\label{eq:ELBO_minibatch} \widehat{\mathcal{L}}(\bmm,S,\btheta)=\frac{n}{B}\sum_{i\in\mathcal{B}} \ell_i(\bmm,S,\btheta)- \mathrm{KL}\!\left[q(\boldsymbol{u})\,\|\,p_{\btheta}(\boldsymbol{u})\right],\end{equation}
i.e.\ the likelihood contribution is reweighted by $n/B$, while the KL term is kept exact since it does not depend on the data. The corresponding stochastic gradients are unbiased, i.e. $\nabla \mathcal{L}=\mathbb{E}_{\mathcal{B}}\bigl[\nabla \widehat{\mathcal{L}}\bigr].$
In practice, one repeatedly samples minibatches $\mathcal{B}$ and performs updates based on $\nabla \widehat{\mathcal{L}}$. The computational cost per iteration is reduced from $\mathcal{O}(n m^2)$ to $\mathcal{O}(B m^2)$. For numerical stability, $S$ is typically parameterized via its Cholesky factor $S = LL^\top$, and natural gradients are used for $(\bmm,S)$, yielding closed-form updates in the variational parameter space. Hyperparameters $\btheta$ and optionally also inducing locations $Z$ are updated with standard stochastic optimizers (e.g.\ Adam), using the same minibatch estimator.

\subsection{Fourier features and spectral (inter-domain) approximations}
A different approximation route replaces the kernel by an explicit finite-dimensional feature map. Random Fourier features (RFF) \citep{Rahimi2007RandomFF} approximate a stationary kernel by $k(x,x')\approx \boldsymbol{\phi}(x)^\top\boldsymbol{\phi}(x')$ with $\boldsymbol{\phi}:\mathcal{X}\to\mathbb{R}^D$, yielding the Bayesian linear model $f(x)=\boldsymbol{\phi}(x)^\top \boldsymbol{w}$ with $\boldsymbol{w}\sim\mathcal{N}(\bm 0,I_D)$. Letting $\Phi=[\boldsymbol{\phi}(x_1);\dots;\boldsymbol{\phi}(x_n)]$, the posterior is $p(\boldsymbol{w}\mid \by)=\mathcal{N}(\boldsymbol{m}_w,S_w)$ with $S_w=(I+\sigma^{-2}\Phi^\top\Phi)^{-1}$ and $\boldsymbol{m}_w=\sigma^{-2}S_w\Phi^\top \by$, and hence $p(f_\ast\mid \by)=\mathcal{N}(\boldsymbol{\phi}(x_\ast)^\top \boldsymbol{m}_w,\;\boldsymbol{\phi}(x_\ast)^\top S_w\boldsymbol{\phi}(x_\ast))$. The sparse spectrum GP of \cite{JMLR:v11:lazaro-gredilla10a} similarly uses a finite spectral expansion, and can be viewed as an inter-domain inducing-variable method in the Fourier domain.

\subsection{Structured kernel interpolation and matrix inversion-free solvers}
\label{subsec:ski-pcg}
A distinct line of work addresses GP scalability not by modifying the probabilistic model or introducing inducing variables, but by exploiting algebraic structure in the kernel matrix and by replacing direct matrix factorizations with iterative linear-algebra solves.

\paragraph{Structured kernel interpolation (SKI)}
Structured kernel interpolation (SKI), also known as KISS-GP \citep{Wilson2015KernelIF}, approximates the covariance matrix by
interpolating kernel evaluations from a set of inducing points arranged on a regular grid.
Let $G=\{g_1,\dots,g_M\}$ denote grid locations and let $W\in\mathbb{R}^{n\times M}$ be a sparse interpolation matrix such that each
row of $W$ contains only a small number of nonzero entries.
The kernel matrix is approximated as
\begin{displaymath}K(X,X) \approx \tilde K(X,X) := W K(G,G) W^\top,\end{displaymath}
where $K(G,G)$ is the kernel evaluated on the grid.
This approximation defines an exact GP with approximate kernel $\tilde K$, whose covariance admits fast matrix--vector products. When the grid $G$ has Cartesian product structure and the kernel is separable, $K(G,G)$ exhibits Kronecker or Toeplitz structure,
allowing matrix--vector products with $K(G,G)$ to be computed in near-linear time.
Even when such structure is not fully present, the sparsity of $W$ ensures that products with $\tilde K(X,X)$ can be evaluated as
\begin{displaymath}\boldsymbol{v} \mapsto W\bigl(K(G,G)(W^\top \boldsymbol{v})\bigr),\end{displaymath}
with cost dominated by sparse interpolation and dense grid-level kernel operations.
From a computational perspective, SKI is well suited to modern hardware as both sparse interpolation matrices and dense grid-level
kernel products can be implemented efficiently on GPUs, enabling scalable GP inference when an accurate interpolation structure
exists.

\paragraph{Matrix inversion-free exact GP inference} This method introduced by \cite{Gardner2018GPyTorchBM} and scaled to GPs of millions of data points by \cite{pcg} keeps the exact kernel $K(X,X)$ and replaces matrix factorizations by iterative,
matrix-free linear algebra.
Exact GP training and prediction involve quantities of the form $(K(X,X)+\sigma^2 I)^{-1}\by$ and $\log\det(K(X,X)+\sigma^2 I)$,
where $\by$ denotes the vector of observations.
Rather than forming or factorizing the matrix $K(X,X)+\sigma^2 I$, these quantities are computed using conjugate gradient (CG)
methods, which solve linear systems
\begin{displaymath}(K(X,X)+\sigma^2 I)\boldsymbol{\alpha} = \by\end{displaymath}
using only repeated matrix--vector products $\boldsymbol{v}\mapsto (K(X,X)+\sigma^2 I)\boldsymbol{v}$.
If such products can be evaluated efficiently, CG yields accurate approximate solutions in a small number of iterations without materializing the kernel matrix.
This matrix-free formulation enables exact GP inference at dataset sizes far beyond the limits of Cholesky-based approaches and is
particularly well suited to GPU implementations, where kernel evaluations can be parallelized over data points
\citep{Gardner2018GPyTorchBM}. To accelerate convergence, one typically employs preconditioned conjugate gradients (PCG), where the system is transformed using a
symmetric positive definite preconditioner $P\approx K(X,X)+\sigma^2 I$ that is cheap to apply, yielding
\begin{displaymath}P^{-1}(K(X,X)+\sigma^2 I)\boldsymbol{\alpha} = P^{-1}\by.\end{displaymath}
Common choices include diagonal Jacobi, low-rank, or Nyström-based approximations.
Both kernel matrix--vector products and preconditioner applications can be implemented in a matrix-free manner and are compatible
with GPU acceleration.

Exact GP training requires evaluation of $\log\det(K(X,X)+\sigma^2 I)$ and its gradients.
Writing $\log\det(A)=\mathrm{tr}(\log A)$ with $A=K(X,X)+\sigma^2 I$, the log-determinant is approximated using stochastic trace estimation combined with Lanczos methods, \begin{displaymath}\log\det(A)\approx \frac{1}{S}\sum_{s=1}^S \boldsymbol{g}_s^\top (\log A)\, \boldsymbol{g}_s,\end{displaymath}
where $\boldsymbol{g}_s$ are independent normal probe vectors and each bilinear form is evaluated via a Lanczos expansion driven solely by
matrix--vector products with $A$.
Gradients of the log marginal likelihood with respect to the hyperparameter vector $\btheta$ are computed componentwise. For the $j$th component, one has
\begin{displaymath}\frac{\partial}{\partial \theta_j}\log p(\by\mid\btheta)=-\tfrac12 \boldsymbol{\alpha}^\top (\partial_{\theta_j} K_{\btheta})\boldsymbol{\alpha}\;+\;\tfrac12 \mathrm{tr}\!\left( K_{\btheta}^{-1}\partial_{\theta_j} K_{\btheta} \right), \qquad \boldsymbol{\alpha}=(K_{\btheta}+\sigma^2 I)^{-1}\by,\end{displaymath}
where the quadratic form is evaluated using the PCG solution $\boldsymbol{\alpha}$ and the trace term is approximated by a Hutchinson estimator, \begin{math}
    \mathrm{tr}(K_{\btheta}^{-1}\partial_{\theta_j} K_{\btheta})\approx
\frac{1}{S}\sum_{s=1}^S \boldsymbol{g}_s^\top K_{\btheta}^{-1}(\partial_{\theta_j} K_{\btheta} \boldsymbol{g}_s),
\end{math} and each product $K_{\btheta}^{-1}\boldsymbol{g}_s$ is again realized by a PCG solve. Thus all terms required for exact GP training and prediction are approximated using only matrix--vector products with $K_{\btheta}$ and $\partial_{\theta_j} K_{\btheta}$.


\section{Equivariant Sparse GPs}\label{sec:eq_gps}
This section establishes the central theoretical observation of the paper, claiming that stochastic equivariance of a Gaussian process is preserved under conditioning. As a consequence, we will see that predictive posterior distributions, sparse and scalable GP inference schemes inherit equivariance automatically whenever they correspond to exact or approximate conditioning of an equivariant prior GP. This result provides the theoretical foundation for constructing large-scale stochastically equivariant GP models. 

\subsection{Conditional stochastic equivariance of GPs}
We begin by formalizing the relationship between stochastic equivariance of a Gaussian process and stochastic equivariance of its conditional distributions. 

\begin{thm}[Equivalences between unconditional and conditional stochastic equivariance of GPs]\label{thm:equivariant-conditioning}
Let $\bff$ be a centred $\mathbb{R}^{p}$-valued Gaussian process on $(\Omega, \mathcal{A}, \mathbb{P})$ indexed by $D$ with kernel $K \colon D \times D \to \R^{p \times p}$. Let a linear group $G$ act on $D$ via $\star$ and on $\R^{p}$ by a representation $\rho \colon g \mapsto \rho_{g} \in \R^{p \times p}$.
Then the following are equivalent:
\begin{enumerate}[label=(\roman*)]
    \item For all
    $\bx \in D$ and $g \in G$,
    \begin{equation*}\mathbb{P}\big(\bff(g \star \bx) = \rho_{g} \bff(\bx)\big) = 1.\end{equation*}
    \item For \textbf{any} finite collection of inputs in $D,$ $X = \left(\bx_{1},\dots,\bx_{n}\right)$ and $\R^{pn}$-valued random vector $\boldsymbol{\varepsilon}$ on $(\Omega, \mathcal{A}, \mathbb{P}),$ for all $\bx \in D$ and $g \in G$,
    \begin{equation*}\mathbb{P}\big(\bff(g \star \bx) = \rho_{g} \bff(\bx) \,\big|\, \bff(X)+\boldsymbol{\varepsilon}\big)= 1 \quad \text{almost surely.}\end{equation*}
    \item For \textbf{some} finite collection of inputs in $D,$ $X = \left(\bx_{1},\dots,\bx_{n}\right),$ and $\R^{pn}$-valued random vector $\boldsymbol{\varepsilon}$ on $(\Omega, \mathcal{A}, \mathbb{P}),$ for all $\bx \in D$ and $g \in G$,
    \begin{equation*}\mathbb{P}\big(\bff(g \star \bx) = \rho_{g} \bff(\bx) \,\big|\, \bff(X)+\boldsymbol{\varepsilon}\big)= 1 \quad \text{almost surely.}\end{equation*}
\end{enumerate}
\end{thm}

\begin{proof}

Fix $\bx\in D$ and $g\in G$, and define the event $E:=\{\bff(g\star\bx)-\rho_g\bff(\bx)=\bm 0\}$. By assumption (i), $\mathbb{P}(E)=1$. Now take any finite collection $X = \left(\bx_{1},\dots,\bx_{n}\right)$ in $D$ and any $\R^{pn}$-valued random vector $\boldsymbol{\varepsilon}$ on $(\Omega, \mathcal{A}, \mathbb{P})$.
Then $\mathbb{P}(E\mid \bff(X)+\boldsymbol{\varepsilon})=1$ almost surely since $\mathbb{E}[\mathbb{P}(E\mid \bff(X)+\boldsymbol{\varepsilon})]=\mathbb{P}(E)=1$ by the law of total probability, and thus $\mathbb{P}(E\mid \bff(X)+\boldsymbol{\varepsilon})<1$ with positive probability would result in a contradiction. This implies (ii) and in turn (iii). Now, assuming (iii), we again obtain by the law of total probability that
\begin{equation}
\mathbb{P}(E)=\mathbb{E}\big[\mathbb{P}(E\mid \bff(X)+\boldsymbol{\varepsilon})\big]=1.
\end{equation}
\end{proof}

While the previous result would remain valid for a broader class of conditioning random elements, the requirements are quite demanding as they involve almost sure statements. The specific form of Gaussian distributions and their associated conditionals makes it actually possible to establish much stronger results. For a start, assuming that $\boldsymbol{\varepsilon}$ is Gaussian and independent of $\bff$, the explicit GP conditioning formulas yield
\begin{equation}
\mathbb{P}\bigl(\bff(g\star \bx)=\rho_g\bff(\bx)\mid \by\bigr)=1
\end{equation}
for any prescribed choice of $\by \in \R^{pn}$. One may then wonder if a sufficient condition of unconditional stochastic equivariance does exist that requires less than almost sure conditional stochastic equivariance. The next results give a positive answer: in particular, for $X$ consisting of one point in $\R^q$, conditional stochastic equivariance for $q+1$ well-chosen conditioning responses is sufficient.



\begin{lem}
\label{lem1}
Let $\boldsymbol{A}$ be a $\R^p$-valued Gaussian vector. If there exists $q\geq 1$, a $\R^q$-valued Gaussian vector $\boldsymbol{B}$ with full-rank covariance matrix $\Sigma_{\boldsymbol{B}}$ and $\boldsymbol{b}_1,\dots,\boldsymbol{b}_{q+1}\in \R^q$ satisfying
$$
\left\{
\begin{array}{l}
(\boldsymbol{A},\boldsymbol{B}) \text{ jointly Gaussian}, \\
\boldsymbol{b}_1-\boldsymbol{b}_{q+1},\dots,\boldsymbol{b}_{q}-\boldsymbol{b}_{q+1} \text{ linearly independent}, \\
\mathbb{P}(\boldsymbol{A}=\bm 0\mid \boldsymbol{B}=\boldsymbol{b}_i)=1 \text{ for } i=1,\dots,q+1,
\end{array}
\right.
$$
then $\mathbb{P}(\boldsymbol{A}=\bm 0)=1$.
\end{lem}
\begin{proof}
For any $i\in\{1,\dots,q+1\}$, $\mathbb{P}(\boldsymbol{A}=\bm 0\mid \boldsymbol{B}=\boldsymbol{b}_i)=1$ implies $\mathbb{E}[\boldsymbol{A}\mid \boldsymbol{B}=\boldsymbol{b}_i]=\bm 0$ and $\operatorname{Var}[\boldsymbol{A}\mid \boldsymbol{B}=\boldsymbol{b}_i]=0$. Since $(\boldsymbol{A},\boldsymbol{B})$ is jointly Gaussian, this gives
\begin{align}
\mathbb{E}[\boldsymbol{A}\mid \boldsymbol{B}=\boldsymbol{b}_i]&=\mathbb{E}[\boldsymbol{A}]+\Sigma_{\boldsymbol{A}\boldsymbol{B}}\Sigma_{\boldsymbol{B}}^{-1}(\boldsymbol{b}_i-\mathbb{E}[\boldsymbol{B}])=\boldsymbol{0},\label{lem1eq1}\\
\operatorname{Var}[\boldsymbol{A}\mid \boldsymbol{B}=\boldsymbol{b}_i]&=\Sigma_{\boldsymbol{A}}-\Sigma_{\boldsymbol{A}\boldsymbol{B}}\Sigma_{\boldsymbol{B}}^{-1}\Sigma_{\boldsymbol{B}\boldsymbol{A}}=0.\label{lem1eq2}
\end{align}
Equating \eqref{lem1eq1} for $i=j\in\{1,\dots,q\}$ and $i=q+1$ yields $\Sigma_{\boldsymbol{A}\boldsymbol{B}}\boldsymbol{v}_j=\bm 0$ with $\boldsymbol{v}_j=\Sigma_{\boldsymbol{B}}^{-1}(\boldsymbol{b}_j-\boldsymbol{b}_{q+1})$. As the vectors $\boldsymbol{b}_j-\boldsymbol{b}_{q+1}$ are linearly independent, so are the $\boldsymbol{v}_j$. Hence $\Sigma_{\boldsymbol{A}\boldsymbol{B}}\boldsymbol{v}_j=\bm 0$ for $j=1,\dots,q$ implies $\Sigma_{\boldsymbol{A}\boldsymbol{B}}=0$. Looking back at \eqref{lem1eq1} and \eqref{lem1eq2}, this implies $\mathbb{E}[\boldsymbol{A}]=\bm 0$ and $\Sigma_{\boldsymbol{A}}=0$.
\end{proof}
The previous lemma extends to non-invertible covariance matrices as follows.
\begin{lem}\label{lem2}
Let $\boldsymbol{A}$ be a $\R^p$-valued Gaussian vector. If there exists $q\geq 1$, a $\R^q$-valued Gaussian vector $\boldsymbol{B}$ possessing a covariance matrix $\Sigma_{\boldsymbol{B}}$ with rank $r\leq q$, and $\boldsymbol{b}_1,\dots,\boldsymbol{b}_{r+1}\in \R^q$ satisfying
$$
\left\{
\begin{array}{l}
(\boldsymbol{A},\boldsymbol{B}) \text{ jointly Gaussian}, \\
\Sigma_{\boldsymbol{B}}^{\dag}(\boldsymbol{b}_1-\boldsymbol{b}_{r+1}),\dots,\Sigma_{\boldsymbol{B}}^{\dag}(\boldsymbol{b}_{r}-\boldsymbol{b}_{r+1}) \text{ linearly independent}, \\
\mathbb{P}(\boldsymbol{A}=\bm 0\mid \boldsymbol{B}=\boldsymbol{b}_i)=1 \text{ for } i=1,\dots,r+1,
\end{array}
\right.
$$
then $\mathbb{P}(\boldsymbol{A}=\bm 0)=1$.
\end{lem}
\begin{proof}
For any $i\in\{1,\dots,r+1\}$, $\mathbb{P}(\boldsymbol{A}=\bm 0\mid \boldsymbol{B}=\boldsymbol{b}_i)=1$ implies $\mathbb{E}[\boldsymbol{A}\mid \boldsymbol{B}=\boldsymbol{b}_i]=\bm 0$ and $\operatorname{Var}[\boldsymbol{A}\mid \boldsymbol{B}=\boldsymbol{b}_i]=0$. Since $(\boldsymbol{A},\boldsymbol{B})$ is jointly Gaussian, this gives
\begin{align}
\mathbb{E}[\boldsymbol{A}\mid \boldsymbol{B}=\boldsymbol{b}_i]&=\mathbb{E}[\boldsymbol{A}]+\Sigma_{\boldsymbol{A}\boldsymbol{B}}\Sigma_{\boldsymbol{B}}^{\dag}(\boldsymbol{b}_i-\mathbb{E}[\boldsymbol{B}])=\bm 0,\label{lem2eq1}\\
\operatorname{Var}[\boldsymbol{A}\mid \boldsymbol{B}=\boldsymbol{b}_i]&=\Sigma_{\boldsymbol{A}}-\Sigma_{\boldsymbol{A}\boldsymbol{B}}\Sigma_{\boldsymbol{B}}^{\dag}\Sigma_{\boldsymbol{B}\boldsymbol{A}}=0.\label{lem2eq2}
\end{align}
Equating \eqref{lem2eq1} for $i=j\in\{1,\dots,r\}$ and $i=r+1$ yields $\Sigma_{\boldsymbol{A}\boldsymbol{B}}\boldsymbol{v}_j=\bm 0$ with $\boldsymbol{v}_j=\Sigma_{\boldsymbol{B}}^{\dag}(\boldsymbol{b}_j-\boldsymbol{b}_{r+1})$. As the vectors $\Sigma_{\boldsymbol{B}}^{\dag}(\boldsymbol{b}_j-\boldsymbol{b}_{r+1})$ are linearly independent, $\Sigma_{\boldsymbol{A}\boldsymbol{B}}\boldsymbol{v}_j=\bm 0$ for $j=1,\dots,r$ implies $\Sigma_{\boldsymbol{A}\boldsymbol{B}}\Sigma_{\boldsymbol{B}}^{\dag}=0$. Looking back at \eqref{lem2eq1} and \eqref{lem2eq2}, this implies $\mathbb{E}[\boldsymbol{A}]=\bm 0$ and $\Sigma_{\boldsymbol{A}}=0$.
\end{proof}

\begin{thm}
\label{thm:finite response equiv}
Let $(\bff(\bx))_{\bx\in D}$ be a $\R^p$-valued GP indexed by a set $D$, and let $(G,\circ)$ be a linear group acting on $D$ via $\star$ and possessing a representation $g\in G \mapsto \rho_g \in \R^{p\times p}$. If there exists $\bx_0 \in D$ and $\by_1,\dots,\by_{r+1} \in \R^{p}$, where $r=\operatorname{rk}\bigl(K(\bx_0,\bx_0)\bigr),$
such that
$$
\left\{
\begin{array}{l}
K(\bx_0,\bx_0)^{\dag}(\by_1-\by_{r+1}),\dots,K(\bx_0,\bx_0)^{\dag}(\by_r-\by_{r+1}) \text{ linearly independent}, \\
\text{and } \mathbb{P}\bigl(\bff(g\star \bx)-\rho_g\bff(\bx)=\bm 0 \mid \bff(\bx_0)=\by_i\bigr)=1 \text{ for } i=1,\dots,r+1,
\end{array}
\right.
$$
then $\mathbb{P}\bigl(\bff(g\star \bx)-\rho_g\bff(\bx)=\bm 0\bigr)=1$.

Thus, if $\mathbb{P}\bigl(\bff(g\star \bx)-\rho_g\bff(\bx)=\bm 0 \mid \bff(\bx_0)=\by_i\bigr)=1$ for any $g \in G$, $\bx\in D$, and for $i=1,\dots,r+1$, then $\mathbb{P}\bigl(\bff(g\star \bx)-\rho_g\bff(\bx)=\bm 0\bigr)=1$ for any $g \in G$, $\bx\in D$.
\end{thm}

\begin{proof}
Set $\boldsymbol{A}_{g\bx}=\bff(g\star \bx)-\rho_g\bff(\bx)$ and $\boldsymbol{B}=\bff(\bx_0)$, and apply Lemma~\ref{lem2}.
\end{proof}

\begin{cor}
Let the assumptions of Theorem~\ref{thm:finite response equiv} hold, and let $\boldsymbol{\varepsilon}_0\sim \mathcal N(\bm 0,\sigma^2 I_p)$ be independent of $\bff$, with $\sigma^2>0$. If there exists $\bx_0 \in D$ and $\by_1,\dots,\by_{p+1} \in \R^{p}$ such that
$$\left\{
\begin{array}{l}
\by_1-\by_{p+1},\dots,\by_p-\by_{p+1} \text{ linearly independent}, \\
\text{and } \mathbb{P}\bigl(\bff(g\star \bx)-\rho_g\bff(\bx)=\bm 0 \mid \bff(\bx_0)+\boldsymbol{\varepsilon}_0=\by_i\bigr)=1 \text{ for } i=1,\dots,p+1,
\end{array}
\right.
$$then $\mathbb{P}\bigl(\bff(g\star \bx)-\rho_g\bff(\bx)=\bm 0\bigr)=1$.

Thus, if $\mathbb{P}\bigl(\bff(g\star \bx)-\rho_g\bff(\bx)=\bm 0 \mid \bff(\bx_0)+\boldsymbol{\varepsilon}_0=\by_i\bigr)=1$ for any $g \in G$, $\bx\in D$, and for $i=1,\dots,p+1$, then $\mathbb{P}\bigl(\bff(g\star \bx)-\rho_g\bff(\bx)=\bm 0\bigr)=1$ for any $g \in G$, $\bx\in D$.
\end{cor}

\begin{proof}
Set $\boldsymbol{A}_{g\bx}=\bff(g\star \bx)-\rho_g\bff(\bx)$ and $\boldsymbol{B}=\bff(\bx_0)+\boldsymbol{\varepsilon}_0$. Then $(\boldsymbol{A}_{g\bx},\boldsymbol{B})$ is jointly Gaussian and
\begin{equation}
\operatorname{Cov}(\boldsymbol{B})=K(\bx_0,\bx_0)+\sigma^2 I_p.
\end{equation}Since $\sigma^2>0$, the matrix $K(\bx_0,\bx_0)+\sigma^2 I_p$ has rank $p$. The claim therefore follows from Lemma \ref{lem1}.
\end{proof}

\subsection{Consequences for sparse and scalable GP inference}
\label{subsec:consequences-sparse}
Theorem~\ref{thm:equivariant-conditioning} allows a unified assessment of equivariance for scalable GP methods. If a method produces posterior predictions by conditioning a stochastically equivariant GP, then the resulting posterior predictive distributions are stochastically equivariant. In particular, we will demonstrate that equivariance is a property of the resulting conditional distribution (and thus of the kernel) and does not depend on how the conditioning set (subset points, inducing points, ELBO, etc.) was selected. Throughout, assume a centred $\mathbb{R}^p$-valued GP $\myf\sim\mathrm{GP}(\bm 0,K)$ on $D$ and a linear group action $\star$ on $D$ with representation $\rho$ on $\mathbb{R}^p$. 

\subsubsection{Subset-of-data and Nyström-type approximations}
\label{subsec:conseq-sod-nystrom}
SoD performs exact GP regression on a subset $X_s\subset X$ with $|X_s|=m$ and the corresponding observation subvector $\by_s$.
Equivalently, it uses the conditional GP $\myf\,|\,\by_s$ and predicts with its posterior distribution.
Therefore, if the prior GP $\myf$ is stochastically equivariant, then the SoD posterior process is stochastically equivariant by
Theorem~\ref{thm:equivariant-conditioning}.
Nyström approximations replace $K(X,X)$ by $\tilde K(X,X)=K(X,Z)K(Z,Z)^{-1}K(Z,X)$ for some $Z\subset D$ with $|Z|=m$.
This can be viewed as defining an {approximate kernel} $\tilde K$ and hence an exact GP model $\tilde{\myf}\sim\mathrm{GP}(\bm 0,\tilde K)$.
Posterior predictions are then those of exact GP regression under the kernel $\tilde K$. Consequently, Nyström posterior distributions are stochastically equivariant whenever the approximate kernel $\tilde K$ satisfies the same kernel equivariance condition \eqref{eq:kernel-equivariance} (with $K$ replaced by $\tilde K$). In particular, equivariance holds if $\tilde K$ is constructed from an equivariant base kernel in a way that preserves \eqref{eq:kernel-equivariance}. In the next subsection we see that this is equivalent to $K$ satisfying \eqref{eq:kernel-equivariance}.

\subsubsection{Inducing-point prior approximations} 
\label{subsec:conseq-prior-inducing}
These methods use inducing locations $Z$, the corresponding inducing vector $\boldsymbol{u}$, and the Nyström term $Q(\bx,\bx'):=K(\bx,Z)K(Z,Z)^{-1}K(Z,\bx').$ It is straightforward to see that $Q$ is an equivariant kernel in the sense of \eqref{eq:kernel-equivariance} if and only if $K$ is equivariant as for any $g,h, \bx, \bx'$ it holds \begin{displaymath}K(g\star \bx,Z)=\rho_gK(\bx,Z), \qquad K(Z,h\star \bx')=K(Z,\bx')\rho_h^\top.\end{displaymath}
Both SoR and DTC can be interpreted as exact GP regression under an approximate prior covariance obtained by replacing parts of $K$ with $Q$.
SoR performs exact conditioning under the GP $\tilde{\myf}\sim\mathrm{GP}(\bm 0,Q).$ 
Hence, by Theorem~\ref{thm:equivariant-conditioning}, the SoR posterior distribution is stochastically equivariant whenever the
exact prior kernel $K$ is equivariant. 
DTC uses the same $Q(X,X)$ and $Q(\bx_\ast,X)$ as SoR in the data-fit and cross-covariance, but retains the exact test prior
$K(\bx_\ast,\bx_\ast)$, yielding a posterior distribution \begin{math}
    p\bigl(\tilde{\myf}(\bx_\ast)\mid \by\bigr)=
\mathcal{N}\!\bigl(Q(\bx_\ast,X)(Q(X,X)+\sigma^2 I)^{-1}\by,\; K(\bx_\ast,\bx_\ast)-Q(\bx_\ast,X)(Q(X,X)+\sigma^2I)^{-1}Q(X,\bx_\ast)\bigr).
\end{math}
Analogously the DTC covariance blocks transform equivariantly, and the DTC posterior predictive distribution is stochastically
equivariant if and only if $K$ satisfies \eqref{eq:kernel-equivariance}.

\paragraph{FITC and PITC}
FITC and PITC modify the Nyström covariance $Q$ by re-inserting the missing marginal (or block-marginal)
variances of the exact kernel $K$. In FITC, one replaces the exact conditional covariance
$K(X,X)-Q(X,X)$ by its (block-)diagonal part, yielding the {FITC condition}
\begin{displaymath}\myf(\bx_i)\perp \myf(\bx_j)\mid \boldsymbol{u} \qquad (i\neq j), \quad\text{equivalently}\quad \mathrm{Cov}(\myf(X)\mid \boldsymbol{u})\approx \operatorname{diag}(K(X,X)-Q(X,X)),\end{displaymath}
where in the vector-valued case $\operatorname{diag}(\cdot)$ denotes $n$ blocks of size $p\times p$.
This can be encoded by the diagonal-corrected Nyström kernel
\begin{displaymath}K_{\mathrm{FITC}}(\bx,\bx') :=Q(\bx,\bx')+\mathbf{1}\{\bx=\bx'\}\big(K(\bx,\bx)-Q(\bx,\bx)\big),\end{displaymath}
On a training set $X$ this satisfies
$K_{\mathrm{FITC}}(X,X)=Q(X,X)+\operatorname{diag}(K(X,X)-Q(X,X))$.
Hence FITC is exact GP regression under the prior $\tilde{\myf}\sim\mathrm{GP}(\bm 0,K_{\mathrm{FITC}}).$ 
Again, if $K_{\mathrm{FITC}}$ is built upon an equivariant kernel $K$ satisfying \eqref{eq:kernel-equivariance}, it follows that $K_{\mathrm{FITC}}$ satisfies \eqref{eq:kernel-equivariance} as well, yielding stochastically equivariant prior and posterior distribtuion by Theorem~\ref{thm:equivariant-conditioning}. PITC is analogous, with the FITC block-diagonal correction replaced by a block-diagonal correction associated with a partition $X=\bigsqcup_b X_b$, corresponding to a {partial} conditional independence assumption within blocks. The resulting model can be written with an analogous block-corrected Nyström kernel obtained by adding $\mathbf{1}\{\bx,\bx'\text{ in same block}\}\big(K(\bx,\bx')-Q(\bx,\bx')\big)$ to $Q(\bx,\bx')$.

\subsubsection{Variational inducing-point method}
\label{subsec:conseq-vfe-svgp}
Variational inducing-point methods approximate the exact posterior process by a tractable Gaussian process obtained from a variational marginal over inducing variables. Fix inducing inputs $Z=\left(\bz_1,\dots,\bz_m\right)$ and the corresponding inducing vector $\boldsymbol{u}$, and consider the
variational family \begin{displaymath}q(\myf)=\int p(\myf\mid \boldsymbol{u})\,q(\boldsymbol{u})\,d\boldsymbol{u},\qquad q(\boldsymbol{u})=\mathcal N(\bmm,S).\end{displaymath}
Here $p(\myf\mid \boldsymbol{u})$ is the exact conditional GP under the prior $\myf\sim\mathrm{GP}(\bm 0,K)$, and the approximation enters
only through $q(\boldsymbol{u})$. For any $\bx\in D$, the conditional distribution is
\begin{displaymath}p\!\left(\myf(\bx)\mid \boldsymbol{u}\right)= \mathcal N\!\big(K(\bx,Z)K(Z,Z)^{-1}\boldsymbol{u},\;K(\bx,\bx)-Q(\bx,\bx)\big).\end{displaymath}
Integrating $\boldsymbol{u}\sim q(\boldsymbol{u})$ yields an induced Gaussian process $q(\myf)=\mathrm{GP}(\bmu_q,K_q)$ with
\begin{displaymath}\bmu_q(\bx)=K(\bx,Z)K(Z,Z)^{-1}\bmm,\end{displaymath}
and
\begin{displaymath}K_q(\bx,\bx')= K(\bx,\bx')+K(\bx,Z)K(Z,Z)^{-1}\big(S-K(Z,Z)\big)K(Z,Z)^{-1}K(Z,\bx').\end{displaymath}
With the prior kernel satisfying  \eqref{eq:kernel-equivariance},
and the Nyström term $Q$ thus inheriting the same equivariance property, the above expressions conjugate as
\begin{displaymath}\bmu_q(g\star \bx)=\rho_g\,\bmu_q(\bx), \qquad K_q(g\star \bx,h\star \bx')=\rho_g\,K_q(\bx,\bx')\,\rho_h^\top.\end{displaymath}
since the middle factor $K(Z,Z)^{-1}(S-K(Z,Z))K(Z,Z)^{-1}$ is independent of $\bx,\bx'.$ Thus for an equivariant prior kernel $K$, the variational posterior process $q(\myf)$ is stochastically equivariant, and so are its predictive marginals.

VFE and SVGP differ in how $(\bmm,S)$ are optimized, 
but they share the same variational family and induced posterior process. Hence both methods preserve stochastic equivariance under the above condition.

\subsubsection{Fourier features and spectral (inter-domain) approximations}
\label{subsec:conseq-rff}
Random-feature and spectral (inter-domain) approximations replace the matrix-valued kernel $K$ by a finite-dimensional
feature surrogate.
In the vector-valued setting, a convenient parametrization is
\begin{displaymath}\tilde K(\bx,\bx')=\Phi(\bx)\,\Phi(\bx')^\top, \qquad \Phi: D\to\mathbb R^{p\times M},\end{displaymath}
which corresponds to the Bayesian linear model
\begin{displaymath}\tilde{\myf}(\bx)=\Phi(\bx)\,\boldsymbol{w},\qquad \boldsymbol{w}\sim\mathcal N(\bm 0,I_M).\end{displaymath}
Hence $\tilde{\myf}\sim\mathrm{GP}(\bm 0,\tilde K)$ is an {exact} Gaussian process, and posterior predictions are obtained by standard GP conditioning with $\tilde K$. By Theorem~\ref{thm:equivariant-conditioning}, the posterior is stochastically equivariant whenever the approximate kernel $\tilde K$ satisfies the equivariance property \eqref{eq:kernel-equivariance}. A sufficient condition is that the feature map $\Phi$ satisfies
\begin{displaymath}\Phi(g\star \bx)=\rho_g\,\Phi(\bx)\qquad \forall g\in G,\ \bx\in D.\end{displaymath}
Indeed, then for all $g,h\in G$,
\begin{displaymath}\tilde K(g\star \bx,h\star \bx') =\Phi(g\star \bx)\Phi(h\star \bx')^\top =\rho_g\,\Phi(\bx)\Phi(\bx')^\top\,\rho_h^\top =\rho_g\,\tilde K(\bx,\bx')\,\rho_h^\top,\end{displaymath}
so $\tilde{\myf}$ has a stochastically equivariant prior and therefore a stochastically equivariant posterior by
Theorem~\ref{thm:equivariant-conditioning}.

\subsubsection{Structured kernel interpolation and matrix-free solvers}
\label{subsec:conseq-ski-pcg}

SKI \cite{Wilson2015KernelIF} and PCG \cite{pcg} mainly exploit computationally convenient structure of the training covariance $K(X,X)$ and GPU-parallelized fast solves of $(K(X,X)+\sigma^2I)\boldsymbol{\alpha}=\by$ to approximate the posterior mean by
\begin{equation*}\boldsymbol{\mu}(\bx_\ast)=K(\bx_\ast,X)\boldsymbol{\alpha},\end{equation*}
resulting in an equivaraint posterior mean whenever $K$ satsifies \eqref{eq:kernel-equivariance}.
With the posterior covariance equivalent to the full GP covariance, the SKI and PCG GPs retain equivariant posterior distributions.

\subsection{Example: SO(2)-Equivariant Sparse Variational GP}
\label{subsec:so2-svgp-example}
This section provides an end-to-end example of stochastically equivariant sparse variational inference. We consider an $\mathbb{R}^2$-valued Gaussian process $\myf=(\myf(\bx))_{\bx\in D}$, with group $G=\mathrm{SO}(2)$
acting on inputs by $\bx\mapsto R\bx$ and on outputs by the standard representation $\rho_R=R$ (i.e.\ vectors rotate as vectors).
The prior is $\myf\sim\mathrm{GP}(\bm 0,K_\Pi)$, where $K_\Pi$ is the  fundamental-region kernel \eqref{eq:fundamental region kernel} built from a projection $\Pi_s(\bx)$ of each input onto a fundamental region $A$ of the action and a section $s(\bx)\in \mathrm{SO}(2)$ that rotates $\bx$ onto $A.$ For $\mathrm{SO}(2)$ on $\mathbb{R}^2$, a canonical fundamental region is the positive $x$-axis
\begin{displaymath}A=\{(r,0): r>0\},\end{displaymath}
so that every $\bx\neq \bm 0$ is mapped to $\Pi_s(\bx)=(\|\bx\|,0)$ by rotating $\bx$ by its polar angle (and we fix $s(\bm 0)=I$).
A key operational consequence is that conditioning the full field on values along $A$ already pins down the entire field. 
Under perfect equivariance of $\myf$, this motivates placing inducing inputs on $A$, because they form a sufficient (in the Titsias sense) set for reconstructing the structure of $\myf$ over the whole orbit space. We hence place inducing inputs on $A$:
\begin{displaymath}Z=\{(r_j,0)\}_{j=1}^m,\qquad r_j\in(0,r_{\max}],\end{displaymath}
using a uniform grid in radius ($r_j$ linearly spaced between $0.01$ and $r_{\max}=\max_i\|\bx_i\|$).
In the experiments below, we use $m \approx 1\%$ of the training set size, i.e.\ aggressive sparsification.
\subsubsection{Evaluation metrics considered}
\label{subsec:metrics}
To assess predictive performance of our GPs, we report both a pointwise error metric and a distributional scoring rule, reflecting the dual role of Gaussian processes as predictors and probabilistic models.
\paragraph{Root mean squared error (RMSE)}
Let $\{(\boldsymbol{x}_i^\ast,\boldsymbol{y}_i^\ast)\}_{i=1}^{n'}$ denote a test set, and let
$\boldsymbol{m}_i = \mathbb{E}[\myf(\boldsymbol{x}_i^\ast)\mid \by] \in \mathbb{R}^p$
be the posterior mean prediction.
The RMSE is defined as
\begin{equation}\label{eq:rmse} \mathrm{RMSE} =\sqrt{ \frac{1}{n'} \sum_{i=1}^{n'} \left\| \boldsymbol{y}_i^\ast - \boldsymbol{m}_i \right\|_2^2 }.\end{equation}
RMSE measures the average Euclidean deviation between posterior mean predictions and observed test targets.
It therefore assesses the quality of point predictions and is insensitive to the posterior uncertainty.

\paragraph{Log score (LogS)}
To assess the quality of the full predictive distribution, we additionally report the log score.
Let
\begin{equation*}\boldsymbol{y}_\ast = \bigl( \boldsymbol{y}_1^\ast;\dots;\boldsymbol{y}_{n'}^\ast \bigr) \in \mathbb{R}^{p n'}\end{equation*}
denote the joint test output vector, and let
\begin{equation*}p(\boldsymbol{y}_\ast \mid \by) = \mathcal{N}(\boldsymbol{m}_\ast,\Sigma_\ast)\end{equation*}
be the joint GP posterior predictive distribution at the test inputs.
The LogS is defined as the negative twice log conditional density evaluated at the observed test outcome,
\begin{equation}\label{eq:logs} \mathrm{LogS} = -2\log p(\boldsymbol{y}_\ast \mid \by)\end{equation}
Unlike the RMSE, which evaluates point predictions through the posterior mean, the LogS assesses the entire joint posterior distribution.
It measures how concentrated the conditional density is around the realized, unseen test outcome and therefore depends on both the posterior mean and covariance, including correlations across output components and test locations.
Low LogS values indicate predictive distributions that are simultaneously accurate and well calibrated.
\subsubsection{Experiment 1}
\label{subsubsec:so2-exp1}
We sample $n=10000$ inputs uniformly on the square $[-100,100]^2$ and simulate a single draw from the equivariant GP prior
$\myf\sim\mathrm{GP}(\bm 0,K_\Pi)$ by forming the full covariance matrix $K_\Pi(X,X)\in\mathbb{R}^{2n\times 2n}$ and multiplying a
standard Gaussian vector by a matrix square root obtained from a singular value decomposition. The kernel $K_\Pi$ is built from the base kernel
\begin{equation*}
    K_{\bar{A}}(\bx,\bx',\btheta)=\begin{pmatrix}
        \sigma_1\exp{\left(-\frac{\|\bx-\bx'\|^2}{2\ell_1^2}\right)} & 0 \\
        0& \sigma_2\exp{\left(-\frac{\|\bx-\bx'\|^2}{2\ell_2^2}\right)}
    \end{pmatrix},
\end{equation*}
with $\btheta=(\sigma_1,\sigma_2,\ell_1,\ell_2)$ sampled uniformly from $[0.01,20]^4.$ The realized field $\by\in\mathbb{R}^{2n}$
is treated as noise-free ground truth. A random subset of $40\%$ of points is used as the training set and the remainder as the test set. Training a full vector-valued GP on the resulting $4000$ training points is already computationally demanding.

We compare:
\begin{enumerate}
    \item a full GP trained by maximum likelihood with Adam for $100$ iterations,
    \item the SVGP trained for $20$ epochs with minibatches of size $B=0.05\,n_{\mathrm{train}}$, using $5$ natural-gradient steps per epoch and $100$ Adam steps on hyperparameters per epoch,
    \item 
    point predictions by the difference of the GP posterior means to the ground truth, 
    \item a measure of uncertainty at the test points, given at each test point $\bx_\ast$ by
    \begin{math}
        u(\bx_\ast)=\tfrac12\log\det\Sigma(\bx_\ast),
    \end{math} where $\Sigma(\bx_\ast)$ is the posterior covariance matrix, and
    \item RMSE, Log scores and the estimated hyperparameters $\btheta.$
\end{enumerate}

The initial values of $\btheta$ used for the optimization are sampled in the smaller hypercube $[0.5,2]^4$ and are the same for both SVGP and the full GP. Figure~\ref{fig:teaser} in the introduction (using a different seed and comparing only the posterior means (also, see Figure~\ref{fig:extra prior samples} in the appendix for more prior samples)), and Figure~\ref{fig:so2-diff-10000} show that the SVGP posterior distribution preserves the global structure of the full GP posterior mean and covariance despite extreme sparsification. Quantitatively, the full GP achieves essentially perfect reconstruction in terms of RMSE on this synthetic draw, whereas the SVGP is less accurate but still competitive in absolute terms ($1.23\times10^{-5}$ vs. $5.49\times10^{-4}$). The probabilistic gap is stronger, as commonly observed for SVGP. The full GP log score is substantially better than the SVGP log score ($-23.49$ vs. $-10.31$), reflecting that the SVGP covariance is only an approximation. The wall-clock times highlight the computational motivation, with the full GP requiring $2720.5$s and the SVGP $111.6$s for training and inference, which corresponds to about a $24$x speedup in this run. This computational advantage allows for more thorough optimization of the kernel parameter vector $\btheta$. We observe that the optimized values $\btheta_{\mathrm{SVGP}}=(3.99,9.11,9.89,10.03)$ are much closer to the actual GP parameters $\btheta=(3.77,9.22,9.92,9.77)$ than the full GP optimized parameters $\btheta_{\mathrm{Full}}=(1.72,1.45,2.15,2.20)$. The latter would eventually move closer to the true parameters, but only at a higher computational cost.

\begin{figure}
\centering
\includegraphics[width=\linewidth]{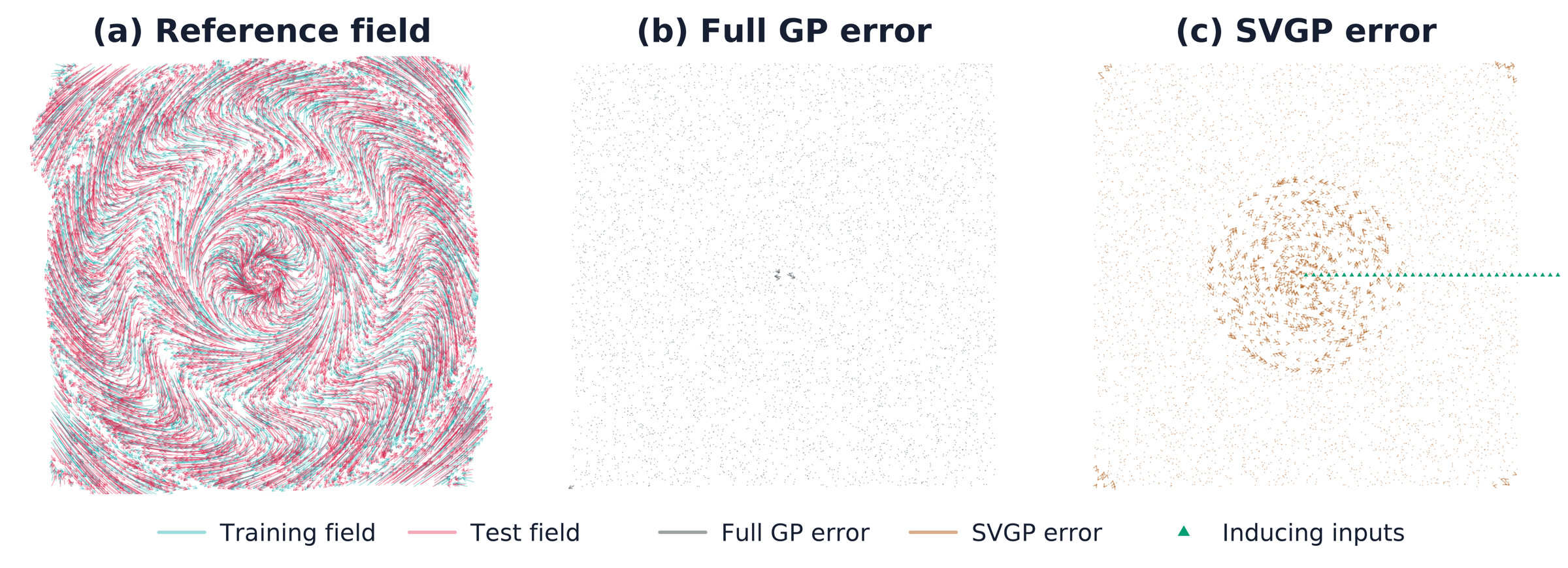}
\par
\noindent\makebox[\linewidth][l]{\includegraphics[width=.92\linewidth]{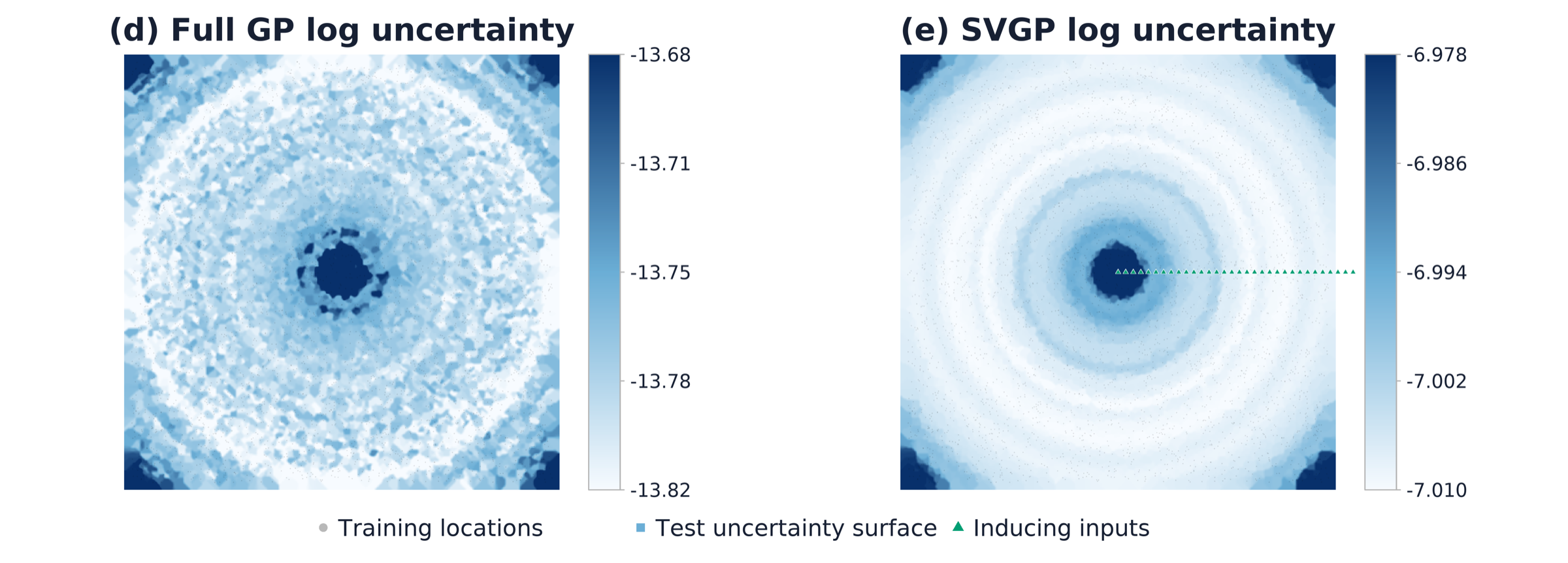}}

\caption{Synthetic $\mathrm{SO}(2)$ experiment with $n=10000$. Top row: realized equivariant field, full GP posterior mean error with respect to the realization, and SVGP posterior mean error. Bottom row: predictive uncertainty of the full GP and SVGP models.} \Description{A multi-panel synthetic SO(2) experiment. The top row shows the sampled equivariant field and posterior mean errors for the full GP and the SVGP. The bottom row shows predictive uncertainty maps for the two methods.} \label{fig:so2-diff-10000} \end{figure}

\subsubsection{Experiment 2
}
\label{subsubsec:so2-exp2}
The second experiment is designed to directly test stochastic equivariance of the posterior distributions and to compare uncertainty structure between full GP and SVGP in a controlled geometry.
We follow the same framework as in Experiment~\ref{subsubsec:so2-exp1} and sample the ground truth from the SO(2)-equivariant GP at $n=6000$ inputs drawn uniformly on the same domain. We define a wedge sector $S$ of opening angle $\pi/6$ in the first quadrant, consisting of points satisfying $\bx^{(1)}\ge 0$, $\bx^{(2)}\ge 0$, with polar angle in $[0,\pi/6]$. We then rotate this wedge by $R_{\pi/3}$ to create a rotated copy $S'=R_{\pi/3}S$. Let $X_S\subset S$ denote the training inputs sampled inside the wedge, and let $\by$ denote the corresponding observed values. The evaluation set consists of additional holdout points inside $S$ whose rotated counterparts lie in $S'$. Posterior equivariance then demands that, for each evaluation location $\bx\in S$,
\begin{align*}
\myf(R\bx)\mid \by&\overset{d}{=}
R\,\myf(\bx)\mid \by,
\qquad R=R_{\pi/3}, \qquad \text{and thus}\\
\mathbb{E}\big[\myf(R\bx)\mid \by\big]
&=R\,\mathbb{E}\big[\myf(\bx)\mid \by\big].
\end{align*}

\begin{figure}
    \centering
    \includegraphics[width=\linewidth]{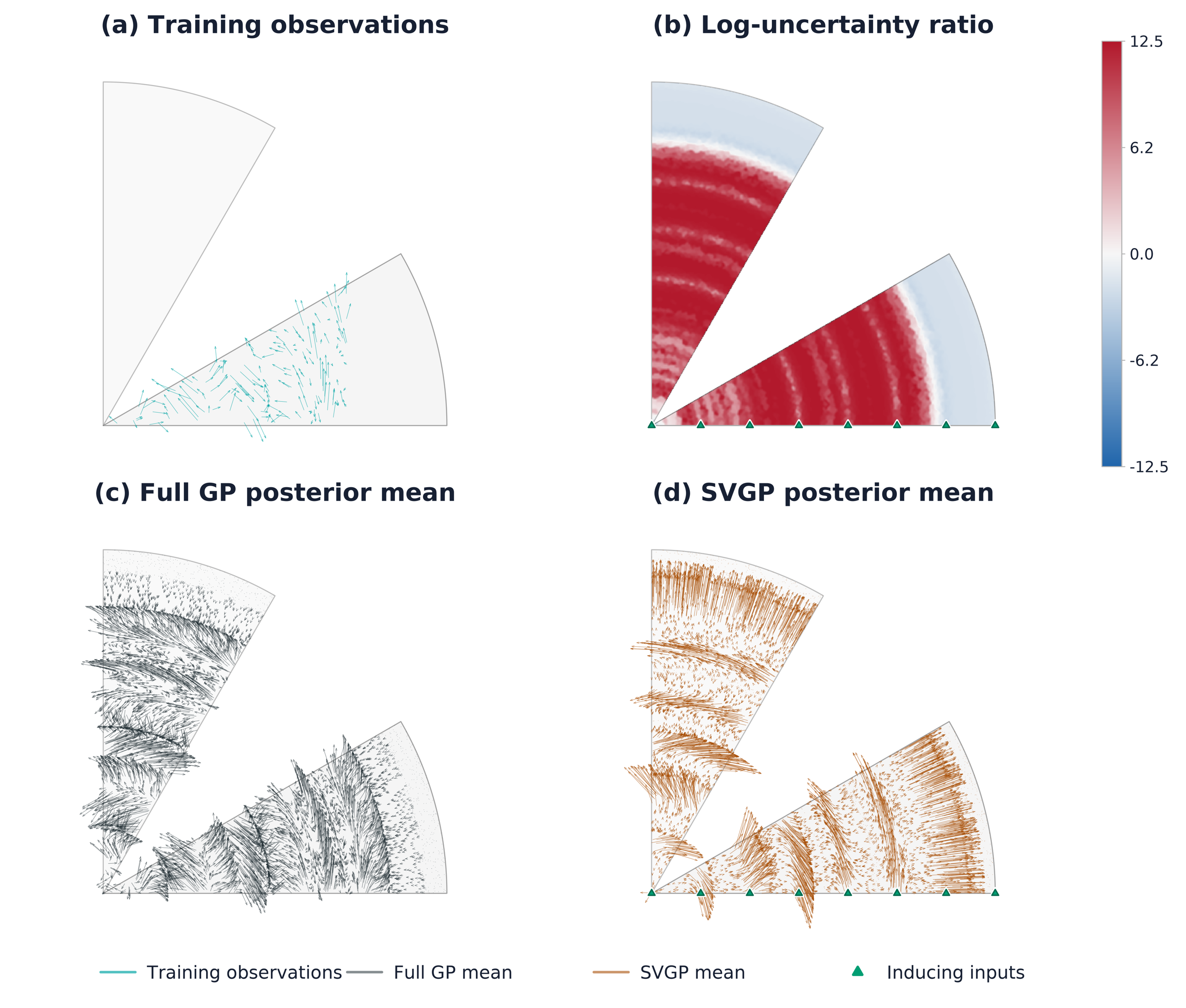}
    \caption{Wedge equivariance check. Top left: training observations in $S$ and rotated copy $S'$. Top right: visualizes $\tfrac12\log\det(\Sigma_{\mathrm{SVGP}})-\tfrac12\log\det(\Sigma_{\mathrm{Full}})$ on test points in $S$ and $S'$.
    Bottom: full GP posterior mean on $S$ and $S'$ (left)
    and SVGP posterior mean on $S$ and $S'$ with inducing points in the fundamental region indicated by green triangles.
    }
    \Description{A four-panel wedge experiment for SO(2) equivariance. The figure shows training observations on a wedge and its rotated copy, a map of uncertainty differences between SVGP and full GP, and posterior mean vector fields for both methods on the two wedges.}
    \label{fig:so2-equiv-wedge}
\end{figure}

By Theorem~\ref{thm:equivariant-conditioning}, the full GP posterior is stochastically equivariant because it is a conditional GP derived from an equivariant kernel. The key point of this experiment is to show that the SVGP posterior remains equivariant, with the equivariance encoded in the posterior mean and covariance despite the approximation. We use the same full GP and SVGP training loop as in Experiment~\ref{subsubsec:so2-exp1}, but with 8 fixed inducing points in $A$. Figure~\ref{fig:so2-equiv-wedge} shows posterior mean vectors for both models on $S$ and $S'$. The SVGP posterior mean approximates the full GP posterior mean on both wedges with less structure and, crucially, the rotated wedge mean field is the rotated version of the original wedge field. In this sense, the SVGP retains the equivariance constraint in the posterior mean, although the approximation is visibly imperfect.
To compare distributional behavior, we visualize the log-uncertainty ratio
\begin{displaymath}u_{\mathrm{SVGP}}(\bxast)-u_{\mathrm{Full}}(\bxast) = \tfrac12\log\det\Sigma_{\mathrm{SVGP}}(\bxast)-\tfrac12\log\det\Sigma_{\mathrm{Full}}(\bxast),\end{displaymath}
over the test points $\bxast$ in $S$ and $S'$. In the shown run, the ratio map indicates that the full GP is typically less uncertain (positive values) in regions well-supported by data, while the SVGP shows the classical overconfidence in extrapolation regimes (negative values).
A particularly interpretable pattern is that uncertainty differences vary across orbits as a function of proximity to inducing support. Locations whose orbit representatives project near inducing radii on $A$ exhibit systematically smaller discrepancy. In those regions, the SVGP tracks the full GP uncertainty more closely. The main message of the two experiments is therefore that aggressive sparsification affects distributional calibration and uncertainty, but not the structural constraint itself, which is an important guarantee.

\section{Triply-scalable equivariant GPs for large molecular datasets}
\label{sec:triply-scalable}
We now apply the integration-free equivariant GP constructions recalled above to large-scale molecular regression. Our target is the electric dipole moment, a $\mathbb{R}^3$-valued quantity representing the imbalance of electron distribution of a molecule, which allows quantum chemists to infer about the interaction of different molecules. Precise techniques exist in chemical laboratories to determine the exact dipole moment of a given molecular configuration. Due to cost and time limitations, quantum chemists are interested in using GP regression as an emulator for the dipole moment. It is known that the dipole moment of a molecule transforms equivariantly under rigid rotations and is invariant to translations of the molecular configuration. Previous work \citep{steinertintegration} established the framework for modeling the dipole moment of water molecules with the equivariant kernel in \eqref{eq:fundamental region kernel}. Building upon this approach, we now move from exact equivariant GP modeling of the water dipole moment to equivariant sparse GPs scalable to larger datasets of dipole moments of arbitrary molecules of fixed type. In our application, we consider a dataset of $21000$ N-methylformamide molecules, each consisting of 9 atoms, $\mathrm{C_2H_5NO}$.


This application requires scalability of equivariant GP regression 
in three complementary aspects. First, equivariant kernel evaluation must remain computationally feasible without numerical group integration. Second, the kernel construction must naturally extend to molecules with arbitrary numbers of atoms. Third, inference and hyperparameter learning must scale to large datasets. We refer to the simultaneous satisfaction of these three requirements as triple scalability.

The following subsections proceed as follows. We first specify the equivariance and invariance constraints arising in molecular dipole prediction and instantiate the fundamental-region kernel construction for general molecular configurations. Next, we describe several inducing-point allocation strategies and the resulting SVGP and PCG inference schemes. Finally, we compare these scalable equivariant GP approximations on the N-methylformamide dataset, reporting learning curves evaluated on test sets of size $1000$.
\newcommand{\bt}{\boldsymbol{t}}
\subsection{Fundamental regions for arbitrary numbers of atoms}
\label{subsubsec:fundamental-region-general}

We first describe the construction of a molecule with $n_a\geq 3$ atoms. A molecular configuration is represented by
\begin{equation*}
    \bx=(\ba_1;\ldots;\ba_{n_a})\in(\mathbb{R}^3)^{n_a}=:D,
\end{equation*}
where $\ba_i\in\mathbb{R}^3$ denotes the position of the molecule's $i$th atom in Euclidean space. We assume that all configurations have the same molecular composition and follow a common atom-indexing convention, so that the chemical element associated with each coordinate block is fixed across the data set. For example, a water configuration is written as
\begin{equation*}
    \bx=(\ba_{\mathrm O};\ba_{\mathrm H_1};\ba_{\mathrm H_2}),
\end{equation*}
where the first block represents oxygen and the remaining two blocks represent the hydrogen atoms. Permutations of atoms of the same type are not included in the present group action and would have to be treated separately if required (see Section \ref{subsec:Reynolds}). The rigid-motion group $SE(3)=SO(3)\ltimes\mathbb{R}^3$ acts on $D$ according to
\begin{equation}
    (R,\bt)\star\bx:=(R\ba_1+\bt;\ldots;R\ba_{n_a}+\bt),\qquad R\in SO(3),\quad \bt\in\mathbb{R}^3.
\end{equation}
The dipole moment is regarded as an unknown nonlinear map $\bmu:D\longrightarrow\mathbb{R}^3$. Physically, it is invariant under translations and equivariant under rotations. Writing
\begin{align}
    \bt\star_1\bx&:=(\ba_1+\bt;\ldots;\ba_{n_a}+\bt),\\
    R\star_2\bx&:=(R\ba_1;\ldots;R\ba_{n_a}),
\end{align}
these properties are
\begin{equation}
    \bmu(\bt\star_1\bx)=\bmu(\bx),\qquad \bmu(R\star_2\bx)=R\bmu(\bx).
\end{equation}

After removing translations, the remaining symmetry group is $G:=SO(3)$, acting on the translation-reduced configurations through $\star_2$. The corresponding output representation is the standard representation $\rho:G\longrightarrow O(3)$, with $\rho_R=R$. To remove the translational degree of freedom, we fix an anchor index $\delta\in\{1,\ldots,n_a\}$ and define
\begin{equation}
    \bar\ba_i:=\ba_i-\ba_\delta,\qquad \bar\bx:=(\bar\ba_1;\ldots;\bar\ba_{n_a}),\qquad \bar\ba_\delta=\bm 0.
\end{equation}
The reduced configuration $\bar\bx$ is invariant under $\star_1$, while the $G$ acts according to
\begin{equation}
    R\star_2\bar\bx:=(R\bar\ba_1;\ldots;R\bar\ba_{n_a}),\qquad R\in G.
\end{equation}

Following the construction for the water molecule in \citep{steinertintegration}, we extend the same idea to molecules with arbitrary numbers of atoms by deterministically selecting two distinguished, displacement vectors $\bv_1(\bar\bx)$ and $\bv_2(\bar\bx)$ from the translation-reduced configuration. The associated fundamental region is defined by requiring the first selected vector to point along a fixed reference axis and the second selected vector to lie in the half-plane:
\begin{equation}
    A:=\left\{\bar\bx:\bv_1(\bar\bx)=r\be_2,\ r>0,\quad \bv_2(\bar\bx)=(a,b,0),\ a>0,\ b\in\mathbb{R}\right\}.
\end{equation}
This is a fundamental region for the set of non-collinear molecules in $D,$ while collinear molecules have non-trivial stabilizers. Since $G$ acts on each atom of the molecule with the same rotation, for $\bx\neq\by \in A,$ it can then be seen that there exists no $R\in G$ such that $R\star_2\bx=\by.$ It is straightforward to show that for each non-collinear translation-reduced configuration $\bx$ there exist $R \in G$ such that $\bx\in R\star_2\bar{A}.$ That is because a projection onto $A$ can be constructed explicitly as a product of three elementary rotations. Two rotations align the first vector with the reference axis, and a third rotation fixes the remaining freedom around that axis by placing the second vector in the prescribed half-plane. Applying the resulting rotation to every atom-position block gives a representative $\Pi(\bar\bx):=s(\bar\bx)\star_2\bar\bx\in A$, where $s(\bar\bx)\in G$ is the corresponding section represented by the afore mentioned rotation. As $A$ is a fundamental region, the section and projection satisfy $\Pi(R\star_2\bar\bx)=\Pi(\bar\bx)$ and $\rho_{s(R\star_2\bar\bx)}=\rho_{s(\bx)}\rho_R^\top$. Consequently, any matrix-valued base kernel on the fundamental region can be lifted to a translation-invariant and $G$-equivariant kernel according to
\begin{equation}\label{eq:Kpi} 
    K_\Pi(\bx,\bx')=\rho_{s(\bar\bx)}^\top K_{\bar A}\bigl(\Pi(\bar\bx),\Pi(\bar\bx')\bigr)\rho_{s(\bar\bx')}.
\end{equation}

In our experiments we take as $K_{\bar A}$ the diagonal matrix-valued squared exponential kernel defined for configurations $\bx,\bx' \in \overline{A}:$
\begin{align}\label{eq:KAforNMF}
K_{\bar A}\left(\bx,\bx';\btheta\right)
&:= \sigma^2\exp\!\left(-\frac{\|\bx-\bx'\|_2^2}{2\ell^2}\right)I_3,
\end{align}
with $\btheta=(\ell,\sigma^2)$. We use a single length scale and variance because the three dipole components transform jointly under rotations and we do not assume different marginal variances across components. Using three different lengthscales and variances did not significantly improve predictive performance in the experiments. 

\subsection{Inducing point methods}
\label{subsec:inducing}
The practical performance of SVGP depends strongly on where inducing points are placed. It is typical to learn the inducing locations $Z$ together with the kernel hyperparameters by minimizing the ELBO \eqref{eq:SVGP elbo}. It is a drawback of the fundamental region kernel $K_\Pi$ for this application that it is discontinuous on $D.$ In our experiments we thus fix an inducing budget as a ratio of the training size and compare several allocation strategies under the same training pipeline.

We choose inducing set sizes as a ratio
\begin{displaymath}M=\lfloor rn\rfloor,\qquad r\in\{0.1,0.2,0.3\},\end{displaymath}
corresponding to the two SVGP regimes reported in the learning curves. The inducing allocation methods evaluated in our learning curves are:
\begin{itemize}
\item \textbf{random:} uniform sampling from the training inputs.
\item \textbf{greedy:} maxi-min sequential greedy selection heuristics based on kernel-induced distances.
\item \textbf{DPP (M-DPP):} an MCMC swap sampler targeting a fixed-size determinantal point process objective \citep{Li2016FastDS}.
\item \textbf{kmeans++:} $k$-means++ seeding using distances in the fundamental region \cite{inproceedings}.
\item \textbf{kmeans++ Lloyd:} $k$-means++ initialisation followed by Lloyd refinement \cite{bachem16fast}.
\item \textbf{recursive RLS:} Nyström dictionary construction via approximate ridge leverage scores \citep{NIPS2017_a03fa308}.
\item \textbf{MMD:} refinement by decreasing the empirical maximum mean discrepancy (kernel herding style).
\end{itemize}


\subsubsection{Greedy selection}
\label{subsubsec:greedy}
Greedy selection aims to build an inducing set that {covers} the dataset under a distance adapted to the kernel. In our implementation the distance is derived from the RKHS geometry of the kernel, using a Hilbert--Schmidt norm. Let $K:\mathcal{X}\times\mathcal{X}\to\mathbb{R}^{3\times 3}$ be a positive definite matrix-valued kernel (here, $K=K_{\Pi}$). Let $\mathcal{H}_K$ be the corresponding vector-valued RKHS. For each input $\bx\in\mathcal{X}$, the kernel defines the kernel feature operator $K_{\bx}:\mathbb{R}^3\to\mathcal{H}_K$ by
\begin{displaymath}K_{\bx} \boldsymbol{v} := K(\cdot,\bx)\,\boldsymbol{v}.\end{displaymath}
The reproducing property gives, for all $\boldsymbol{v},\boldsymbol{w}\in\mathbb{R}^3$,
\begin{displaymath}\langle K_{\bx} \boldsymbol{v}, K_{\bx'} \boldsymbol{w}\rangle_{\mathcal{H}_K} = \boldsymbol{v}^\top K(\bx,\bx') \boldsymbol{w}.\end{displaymath}
We define a point-to-point distance by the Hilbert--Schmidt norm of the difference of these operators:
\begin{displaymath}d(\bx,\bx')^2 := \|K_{\bx} - K_{\bx'}\|_{\mathrm{HS}}^2 = \sum_{i=1}^3 \|K_{\bx} \boldsymbol{e}_i - K_{\bx'} \boldsymbol{e}_i\|_{\mathcal{H}_K}^2,\end{displaymath}
where $(\boldsymbol{e}_i)_{i=1}^3$ is the standard basis of $\mathbb{R}^3$.
Expanding and using the reproducing identity yields
\begin{align*}
d(\bx,\bx')^2
&= \sum_{i=1}^3 \Bigl(
\langle K_{\bx} \boldsymbol{e}_i, K_{\bx} \boldsymbol{e}_i\rangle
+
\langle K_{\bx'} \boldsymbol{e}_i, K_{\bx'} \boldsymbol{e}_i\rangle
-
2\langle K_{\bx} \boldsymbol{e}_i, K_{\bx'} \boldsymbol{e}_i\rangle
\Bigr)\\
&= \sum_{i=1}^3 \Bigl(
\boldsymbol{e}_i^\top K(\bx,\bx) \boldsymbol{e}_i
+
\boldsymbol{e}_i^\top K(\bx',\bx') \boldsymbol{e}_i
-
2 \boldsymbol{e}_i^\top K(\bx,\bx') \boldsymbol{e}_i
\Bigr)\\
&= \mathrm{tr}\bigl(K(\bx,\bx)\bigr) + \mathrm{tr}\bigl(K(\bx',\bx')\bigr) - 2\,\mathrm{tr}\bigl(K(\bx,\bx')\bigr).
\end{align*}
This is the distance used in the greedy procedure. We start from a uniformly sampled seed set containing 5\% of the training set, estimate kernel hyperparameters on that subset, and then add points whose distance from the current set exceeds an adaptive threshold based on a high quantile of within-set distances (here the 0.9 quantile). See Algorithm~\ref{alg:greedy} for the exact procedure.

\begin{algorithm}
\caption{Greedy inducing point selection}
\label{alg:greedy}
\begin{algorithmic}[1]
\State \textbf{Input:} training inputs $X$, target inducing size $M$
\State \textbf{Initialise:} sample a seed set $X_m \subset X$ with $m_0 = 0.05n$
\State Optimise kernel hyperparameters $\btheta$ on $X_m$ by maximising the ELBO (Adam initialisation as $\btheta_0=(1,1,0.1)$)
\State Compute pairwise kernel-induced distances on $X_m$ using
\begin{displaymath}d(\bx,\bx')^2 = \mathrm{tr}\bigl(K(\bx,\bx)\bigr) + \mathrm{tr}\bigl(K(\bx',\bx')\bigr) - 2\,\mathrm{tr}\bigl(K(\bx,\bx')\bigr),\end{displaymath}
and store the within-set distance matrix $D_m$
\While{$|X_m| < M$}
    \For{$\bx \in X \setminus X_m$}
        \State $d_{\min}(\bx) \gets \min_{\bz\in X_m} d(\bx,\bz)$
        \If{$d_{\min}(\bx) > q_{0.9}(D_m)$}
            \State $X_m \gets X_m \cup \{\bx\}$
            \State Recompute/update $D_m$
        \EndIf
    \EndFor
\EndWhile
\State \textbf{Output:} inducing set $X_m$
\end{algorithmic}
\Description{Pseudocode for greedy inducing-point selection. Starting from a random seed set, the algorithm fits kernel hyperparameters and repeatedly adds training inputs that exceed an adaptive kernel-distance threshold until the target size is reached.}
\end{algorithm}


\subsubsection{M-DPP}
\label{subsubsec:mdpp}
Determinantal point process (DPP) objectives encourage diverse subsets by favouring sets with large Gram determinant. The M-DPP sampler used here \eqref{alg:mdpp} is an MCMC swap chain that maintains a fixed-size subset and accepts swaps according to the determinant ratio \citep{Li2016FastDS}. In practice this yields strong diversity at a moderate computational cost for the inducing sizes considered.

\begin{algorithm}
\caption{M-DPP inducing point selection \citep{Li2016FastDS}}
\label{alg:mdpp}
\begin{algorithmic}[1]
\State \textbf{Input:} training set $X$, target inducing size $M$, iterations $T$
\State Initialise $X_m \subset X$ with $M$ random points
\State Optimise kernel hyperparameters $\btheta$ by Adam on the ELBO (as in SVGP training)
\For{$t = 1\colon T$}
    \State Sample $\bx_{1} \in X \setminus X_m$, $\bx_{0} \in X_m$
    \State $Z_1 \gets (X_m \setminus \{\bx_{0}\}) \cup \{\bx_{1}\}$, \quad $Z_0 \gets X_m$
    \State Compute $\det_1 = \det(K(Z_1,Z_1))$, \quad $\det_0 = \det(K(Z_0,Z_0))$
    \State $p \gets \min\{1, \det_1 / \det_0\}$
    \If{Bernoulli$(p)=1$}
        \State $X_m \gets Z_1$
    \EndIf
\EndFor
\State \textbf{Output:} $X_m$
\end{algorithmic}
\Description{Pseudocode for fixed-size determinantal point process sampling. At each iteration, the method proposes replacing one selected input and accepts the swap according to the ratio of kernel Gram determinants.}
\end{algorithm}


\subsubsection{Recursive RLS Nyström}
\label{subsubsec:rls}
Recursive ridge leverage score (RLS) sampling \cite{NIPS2017_a03fa308} constructs a Nyström dictionary by approximating ridge leverage scores that quantify the diagonal of the residual kernel $K(X,X)$ relative to its current Nyström approximation, and sampling candidate inducing points proportionally to these scores. Intuitively, a large score indicates that the current Nyström approximation explains the kernel variance at this point poorly, so including it in the inducing set is expected to yield the greatest improvement. The recursive formulation enables these scores to be estimated on progressively smaller reference sets, making the approach computationally attractive for large $n$.

\begin{algorithm}
\caption{Recursive RLS Nyström inducing point selection}
\label{alg:rls}
\begin{algorithmic}[1]
\State \textbf{Input:} training set $X$, target number $M$
\State Sample initial subset $X_m \subset X$ (at least 3 points)
\State Optimise kernel hyperparameters $\btheta$ via ELBO 
\Function{ApproxRidgeLeverage}{$X', S$}
    \State Compute $C = K(X', S)$ and $W = K(S, S)$
    \State For each $\bx\in X'$, compute $\tau(\bx) = \mathrm{tr}\!\left(K(\bx,\bx) - C_{\bx} W^{-1} C_{\bx}^\top\right)$
    \State \Return normalised scores $\tau / \sum_{\bx\in X'} \tau(\bx)$
\EndFunction
\Function{RLSRecursive}{$X', M$}
    \If{$|X'| \le \lceil 4 M \log(\max(2,M)) \rceil$}
        \State Sample subset of $X'$ proportional to approximate leverage scores
        \State \Return subset
    \EndIf
    \State Randomly select half of $X'$: $X_{\mathrm{ref}} \subset X'$
    \State $S \gets$ RLSRecursive($X_{\mathrm{ref}}, M$)
    \State Compute leverage scores $\tau = \text{ApproxRidgeLeverage}(X', S)$
    \State Sample subset of $X'$ proportional to $\tau$
    \State \Return sampled subset
\EndFunction
\State $X_{\mathrm{over}} \gets$ RLSRecursive($X, M$)
\State Select final $M$ points from $X_{\mathrm{over}}$ (e.g.\ by column norms)
\State \textbf{Output:} inducing set $X_m$
\end{algorithmic}
\Description{Pseudocode for recursive ridge-leverage-score Nystr\"om selection. The method recursively builds reference subsets, estimates residual kernel leverage scores, and samples a final inducing dictionary.}
\end{algorithm}


\subsubsection{$k$-means++ and Lloyd refinement}
\label{subsubsec:kmeans}
$k$-means++ \cite{inproceedings} places representatives by biased seeding, i.e. points far from existing centroids are sampled with higher probability. In our setting we compute distances in the fundamental region (i.e.\ on $\Pi(\bar \bx)$), so that clustering is performed in a translation- and rotation-normalised geometry. Lloyd refinement \cite{bachem16fast} then iteratively updates centroids to local minimisers of the within-cluster squared distance objective.

\begin{algorithm}
\caption{$k$-means++ inducing point selection} 
\label{alg:kmeanspp}
\begin{algorithmic}[1]
\State \textbf{Input:} training set $X$, target number of centroids $M$
\State Sample one random point $\bx$ from $X$ as the first centroid
\For{$m = 2:M$}
    \State Compute squared distances from all points in $X$ to the nearest centroid using Euclidean distance on $\Pi(\bar \bx)$
    \State Set sampling probabilities proportional to these squared distances
    \State Sample a new centroid from $X$ according to these probabilities
    \State Add the new centroid to the centroid set
\EndFor
\State \textbf{Output:} centroids (used as inducing locations)
\end{algorithmic}
\Description{Pseudocode for $k$-means++ seeding. New centroids are sampled with probabilities proportional to squared distance from the current centroid set in fundamental-region coordinates.}
\end{algorithm}


\subsubsection{MMD based refinement}
\label{subsubsec:mmd}
The maximum mean discrepancy (MMD) based method explicitly targets distribution matching. It refines an inducing set $Z$ so that the empirical distribution on $Z$ approximates that of the full dataset $X$ in the kernel mean embedding induced by $K_\Pi$. The procedure below is a simple stochastic swap heuristic that accepts replacements with a probability proportional to the reduction of $\mathrm{MMD}^2_K(X,Z)$. A fast computation of this reduction is given in Appendix~\ref{subsec:MMD-comp}.

\begin{algorithm}[H]
\caption{MMD-based randomised inducing point refinement}
\label{alg:mmd-refine}
\begin{algorithmic}[1]
\State \textbf{Input:} dataset $X$, initial inducing set $Z$ with $|Z|=M$, kernel $K$, parameter $\lambda>0$, iterations $T$
\For{$t=1$ to $T$}
    \State Sample $\bx \sim \mathrm{Unif}(X)$ and $\bz \sim \mathrm{Unif}(Z)$
    \State Set $Z'=(Z\setminus \{\bz\}) \cup \{\bx\}$
    \State Compute $\Delta = \mathrm{MMD}^2_K(X,Z) - \mathrm{MMD}^2_K(X,Z')$
    \If{$\Delta > 0$}
        \State Accept with probability $p = 1 - \exp(-\lambda \Delta)$
        \If{accepted} $Z \gets Z'$ \EndIf
    \EndIf
\EndFor
\State \textbf{Output:} refined inducing set $Z$
\end{algorithmic}
\Description{Pseudocode for MMD-based inducing-point refinement. Random swaps are proposed and accepted probabilistically when they reduce the squared maximum mean discrepancy between the full and inducing input distributions.}
\end{algorithm}

Notice that the acceptance probability depends on the parameter $\lambda\geq0.$ Larger values of $\lambda$ correspond to higher acceptance probabilities. In our experiments we set the default to $\lambda=1.$

\subsection{Implementation}
\label{subsec:implementation}

We compare the full equivariant GP, the equivariant SVGP, and PCG-based matrix-free inference under the kernel $K_\Pi$ in \eqref{eq:Kpi}. The full GP and PCG use the same covariance model, with PCG replacing the dense linear-algebra operations by iterative matrix-free solves. The SVGP uses the variational approximation introduced in Section~\ref{sec:sparse-gps} with inducing inputs selected by the methods of Section~\ref{subsec:inducing}. The hyperparameters of the three approaches are estimated separately using the procedures specified below. The data set contains $21000$ N-methylformamide configurations with three-dimensional dipole responses. For each replicate, we sample a test set of $n'=1000$ configurations without replacement. From the remaining observations, we construct 13 nested training sets of sizes between $100$ and $20000.$ The first training set contains $100$ observations, and each subsequent set is obtained by augmenting the preceding set with previously unused observations. This produces a genuine learning curve within each replicate, since observations are added rather than resampled independently at every training size. The complete procedure is repeated for $50$ independently seeded train--test splits (replicates). At every training size, predictive accuracy and distributional quality are measured using the RMSE and joint LogS defined in \eqref{eq:rmse} and \eqref{eq:logs}, respectively.

The base kernel $K_{\bar A}$ is parameterized as in \eqref{eq:KAforNMF}. All dense kernel matrices used during optimization and prediction receive a diagonal regularization of $10^{-8}.$  For the small-data ($n\leq2000$) comparison with the non-equivariant base GP, the same optimization and prediction procedures are used after replacing $K_\Pi$ by the diagonal squared exponential kernel on the original molecular coordinates. 
\subsubsection{Full GP implementation}
The full GP estimates the parameter vector
\begin{equation*}
\btheta_{\mathrm{full}}=(\ell,\sigma)
\end{equation*} by minimizing the negative log marginal likelihood \eqref{eq:likelihood}. We use Adam with initialization $(\ell_0,\sigma_0)=(1,1)$ and learning rate $0.01$. Separate fits are performed with optimization budgets in form of Adam iterations in $\{10,250,1000\}.$ For $n\leq1000$, the likelihood and its gradient are evaluated on the complete training set. For $n>1000$, parameter estimation uses a reproducible subset of $1000$ training configurations at each iteration. The resulting parameters are subsequently used to evaluate the full-GP posterior using all $n$ training observations. 
\subsubsection{SVGP implementation}
To compute the learning curves for the SVGP, for each inducing ratio $r$, we set the size of the inducing set
\begin{equation*}
M=\lfloor rn\rfloor,\qquad r\in\{0.1,0.2,0.3\}.
\end{equation*}
The main learning-curve comparison among all inducing-point methods uses $r=0.1$. The ratios $r=0.2$ and $r=0.3$ are additionally considered in the small-data and computational-resource comparisons. Once selected, the inducing inputs are kept fixed throughout variational training. The SVGP parameter vector is $\btheta_{\mathrm{SVGP}}=(\ell,\sigma,\sigma_n)$, where $\sigma_n$ is the additional standard deviation of the required noise of the variational distribution. The variational distribution $q(\boldsymbol{u})=\mathcal{N}(\bmm,S)$ is initialized with $\bmm=\boldsymbol{0}$ and $S=K(Z,Z)$. Equivalently, in the natural-parameter implementation, the initial natural mean is zero and the initial precision is $K(Z,Z)^{-1}$. Training consists of ten epochs. Within each epoch, we first perform ten natural-gradient updates of $(\bmm,S)$ with step size $\rho=0.2$. These are followed by ten Adam updates of $(\ell,\sigma,\sigma_n)$ based on the minibatch ELBO \eqref{eq:ELBO_minibatch}. Adam uses learning rate $0.01$. The kernel and likelihood parameters are initialized at $(\ell_0,\sigma_0,\sigma_{n,0})=(1,1,0.1).$

The likelihood minibatches are sampled from the non-inducing training observations. Their size is
\begin{equation*}
B=\max\left\{10,\left\lfloor b(n-M)\right\rfloor\right\},\qquad b=0.2,
\end{equation*}
and their likelihood contribution is rescaled by $(n-M)/B$. The additional minibatch study reported in the appendix uses $b\in\{0.025,0.05,0.2\}$ with all remaining SVGP settings unchanged.

The inducing inputs are selected according to the random, greedy, M-DPP, recursive-RLS, $k$-means++, $k$-means++ with Lloyd refinement, and MMD methods specified in Section~\ref{subsec:inducing}. Each method returns exactly $M$ distinct training configurations. For the four kernel-dependent methods greedy, M-DPP, recursive RLS, and MMD, the kernel parameters entering the selection criterion are obtained from a selector-specific preliminary SVGP fit. The preliminary inducing sets are initialized as specified by the corresponding algorithms. Greedy uses a uniformly sampled seed set of size $\lfloor0.05n\rfloor,$ M-DPP uses its initial uniformly sampled size-$M$ subset, recursive RLS uses three initial configurations and MMD uses a separate uniformly sampled set of size $\lfloor0.05n\rfloor$ for parameter estimation and an independently initialized size-$M$ set for the MMD refinement. For $n\leq1000$, this preliminary fit uses the complete partition into preliminary inducing and remaining training observations. For $n>1000$, it uses at most $1000$ configurations, of which at most $200$ are preliminary inducing inputs. The retained configurations are selected deterministically from the ordered training set. The preliminary SVGP is initialized at $(1,1,0.1)$ and trained for ten epochs, each containing ten natural-gradient updates followed by ten Adam updates with learning rate $0.01$. Its fitted values of $(\ell,\sigma)$ parameterize the corresponding selection criterion and initialize the main SVGP fit. 

The numerical settings of the allocation methods are as follows. Greedy selection uses the $0.9$ quantile threshold specified in Algorithm~\ref{alg:greedy}. M-DPP performs $T=1000$ swap proposals in Algorithm~\ref{alg:mdpp}. Recursive RLS uses oversampling factor $3$ in Algorithm~\ref{alg:rls}. MMD uses $T=1000$ proposals and $\lambda=1$ in Algorithm~\ref{alg:mmd-refine}. Random selection samples uniformly without replacement. For $K_\Pi$, $k$-means++ seeding and Lloyd refinement operate on the fundamental-region coordinates $\Pi(\bar\bx)$. For the standard base-kernel comparison, they operate on the original input coordinates. Lloyd refinement uses at most $1000$ iterations and terminates when the maximum absolute centroid displacement is below $10^{-6}$. Empty clusters are reinitialized with the configurations having the largest current assignment error. Since the SVGP implementation represents inducing inputs by training-set indices, each final centroid is associated with a distinct training configuration from its cluster.

\subsubsection{Equivariant PCG-GP implementation}
For the PCG implementation, we exploit the scalar-diagonal structure of the base kernel \eqref{eq:KAforNMF}. Define
\begin{equation*}
k_{\btheta}(\boldsymbol{u},\boldsymbol{v}):=\sigma^2\exp\!\left(-\frac{\|\boldsymbol{u}-\boldsymbol{v}\|_2^2}{2\ell^2}\right),\qquad K_{\bar A}(\boldsymbol{u},\boldsymbol{v};\btheta)=k_{\btheta}(\boldsymbol{u},\boldsymbol{v})I_3.
\end{equation*}
For each training configuration, let $\rho_i:=\rho_{s(\bar\bx_i)}$. The $(i,j)$ block of the equivariant covariance matrix is then
\begin{equation*}
\left[K_\Pi(X,X;\btheta)\right]_{ij}=k_{\btheta}\!\left(\PiA(\bar\bx_i),\PiA(\bar\bx_j)\right)
\rho_i^\top\rho_j.
\end{equation*}
Let
\begin{equation*}
K_{XX}^{(0)}:=\left[k_{\btheta}\!\left(\PiA(\bar\bx_i),\PiA(\bar\bx_j)\right)\right]_{i,j=1}^n,\qquad
B_{\btheta}:=K_{XX}^{(0)}+\delta I_n,
\end{equation*}
where $\delta>0$ is a numerical diagonal regularisation, and define $U:=\operatorname{blockdiag}(\rho_1,\ldots,\rho_n)$. Since every $\rho_i$ is orthogonal, we can write
\begin{equation*}
K_\Pi(X,X;\btheta)+\delta I_{3n}=U^\top\left(B_{\btheta}\otimes I_3\right)U.
\end{equation*}

Let $\widetilde Y\in\mathbb{R}^{n\times3}$ contain the locally rotated training responses, with $i$th row $\rho_i\boldsymbol{y}_i$. The original $3n$-dimensional linear system is then equivalent to
\begin{equation}
B_{\btheta}\widetilde{ A}=\widetilde Y,\qquad\widetilde{ A}\in\mathbb{R}^{n\times3}.
\label{eq:fast pcg three rhs}
\end{equation}
Hence, only one scalar $n\times n$ kernel operator is required, with the three output coordinates treated as simultaneous right-hand sides. The kernel parameters $(\ell,\sigma)$ are estimated using at most $2000$ observations from the corresponding nested training set. On this subset $H$, the exact negative log marginal likelihood in the rotated frame is
\begin{equation*}
\widetilde{\mathcal L}_{H}(\btheta)=\frac12\left[\left\langle\widetilde Y_H,B_{\btheta,H}^{-1}\widetilde Y_H\right\rangle_F+3\log|B_{\btheta,H}|+3|H|\log(2\pi)\right],
\end{equation*}
where $\langle A,B\rangle_F:=\operatorname{tr}(A^\top B)$. Because $|H|\leq2000$, this objective is evaluated using a dense Cholesky factorisation of the scalar matrix $B_{\btheta,H}$. Optimisation is initialized at $(\ell,\sigma)=(1,1)$, using at most $200$ iterations. The resulting parameters are subsequently used for inference on the complete training set. For $n\leq2000$, Equation~\eqref{eq:fast pcg three rhs} is solved directly by dense Cholesky factorisation. For larger training sets, products with $B_{\btheta}$ are evaluated in double precision using PCG on an NVIDIA H100 GPU. The PCG uses a pivoted-Cholesky as preconditioner improving convergence of the iterative solver. The final posterior mean solve uses relative residual tolerance $10^{-7}$ and at most $4000$ iterations. For test configurations $X^\ast=(\bx_1^\ast,\ldots,\bx_{n_\ast}^\ast)$, define
\begin{equation*}
K_{\ast X}^{(0)}:=\left[k_{\btheta}\!\left(\PiA(\bar\bx_i^\ast),\PiA(\bar\bx_j)\right)
\right]_{\substack{i=1,\ldots,n_\ast\\j=1,\ldots,n}}.
\end{equation*}
The posterior mean in the local test frames is \begin{equation*}
\widetilde M_\ast=K_{\ast X}^{(0)}\widetilde{ A}.
\end{equation*}
Writing $\rho_i^\ast:=\rho_{s(\bar\bx_i^\ast)}$, the prediction in the original Cartesian frame is recovered as \begin{equation*}
\boldsymbol{m}_i^\ast=(\rho_i^\ast)^\top
(\widetilde M_\ast)_{i,:}^{\top},
\qquad
i=1,\ldots,n_\ast.
\end{equation*}
The same frame reduction is used to evaluate the complete joint posterior covariance required by the LogS in \eqref{eq:logs}. Define
\begin{align*}
K_{X\ast}^{(0)}&:=\left(K_{\ast X}^{(0)}\right)^\top,\\ K_{\ast\ast}^{(0)}
&:=\left[k_{\btheta}\!\left(\PiA(\bar\bx_i^\ast),\PiA(\bar\bx_j^\ast)\right)\right]_{i,j=1}^{n_\ast}.
\end{align*} The scalar conditional covariance at all test configurations is then
\begin{equation}S_\ast=K_{\ast\ast}^{(0)}+\delta I_{n_\ast}-K_{\ast X}^{(0)}B_{\btheta}^{-1}K_{X\ast}^{(0)}.
\label{eq:pcg scalar posterior covariance}
\end{equation}
We use $\delta=10^{-6}$ for numerical stability in the score computation.

Let \begin{equation*}
U_\ast:=\operatorname{blockdiag}\left(\rho_1^\ast,\ldots,\rho_{n_\ast}^\ast \right).
\end{equation*}
The covariance of the complete stacked test vector in the original Cartesian frames is
\begin{equation} \Sigma_\ast=U_\ast^\top\left(S_\ast\otimes I_3\right)U_\ast.
\label{eq:pcg full posterior covariance}
\end{equation}
Thus, $S_\ast$ retains all posterior correlations among the test configurations, while $U_\ast$ recovers the corresponding covariance between Cartesian output components. In particular, Equation~\eqref{eq:pcg full posterior covariance} represents the same joint $3n_\ast$-dimensional posterior covariance as direct conditioning with the vector-valued kernel $K_\Pi$.

To compute \eqref{eq:pcg scalar posterior covariance}, we solve
\begin{equation}B_{\btheta}V=K_{X\ast}^{(0)}
\label{eq:pcg covariance solves}
\end{equation}
and form
\begin{equation*}
S_\ast=K_{\ast\ast}^{(0)}+\delta I_{n_\ast}-K_{\ast X}^{(0)}V.
\end{equation*}
For $n\leq2000$, Equation~\eqref{eq:pcg covariance solves} is evaluated by dense scalar Cholesky solves. For larger $n$, its $n_\ast$ right-hand sides are processed in blocks of $25$ using the same matrix-free kernel operator and pivoted-Cholesky preconditioner as above. These covariance solves use relative residual tolerance $10^{-8}$ and at most $6000$ iterations. With $n_\ast=1000$, this requires $40$ blocks of PCG solves. Neither the $3n\times3n$ training covariance nor the $3n_\ast\times3n_\ast$ posterior covariance is constructed and only the final scalar matrix $S_\ast\in\mathbb{R}^{1000\times1000}$ is stored and factorised.

Let $\widetilde E_\ast\in\mathbb{R}^{n_\ast\times3}$ contain the posterior residuals in the local test frames, with rows
\begin{equation*}
(\widetilde E_\ast)_{i,:}=\rho_i^\ast\left(\boldsymbol{y}_i^\ast-\boldsymbol{m}_i^\ast
\right).
\end{equation*}
Orthogonality of $U_\ast$ and the Kronecker structure in \eqref{eq:pcg full posterior covariance} imply
\begin{align*}
q_\ast:=(\boldsymbol{y}_\ast-\boldsymbol{m}_\ast)^\top\Sigma_\ast^{-1}(\boldsymbol{y}_\ast-\boldsymbol{m}_\ast)
&=\left\langle\widetilde E_\ast,S_\ast^{-1}\widetilde E_\ast\right\rangle_F,\\\log|\Sigma_\ast|
&=3\log|S_\ast|.
\end{align*}
After computing the Cholesky factorisation $S_\ast=L_\ast L_\ast^\top$, the Mahalanobis and log-determinant terms are then
\begin{equation*}
q_\ast=\|L_\ast^{-1}\widetilde E_\ast\|_F^2,\qquad\log|\Sigma_\ast|=6\sum_{i=1}^{n_\ast}\log(L_{\ast,ii}).
\end{equation*}
We thus obtain a fast and memory-efficient evaluation of the LogS as
\begin{equation}\mathrm{LogS}=q_\ast+\log|\Sigma_\ast|+3n_\ast\log(2\pi).
\end{equation}

\begin{rem}[Posterior equivariance under approximate covariance solves]
Since Equation~\eqref{eq:pcg covariance solves} is solved numerically, posterior stochastic equivariance could in principle be violated if test configurations related by a group transformation received different approximate solutions. This could occur, for example, if the corresponding PCG solves used independent random initial values, different preconditioners, or different stopping rules.

In our implementation, the $j$th column of $V$ is obtained from
\begin{displaymath}
B_{\btheta}V_{:,j}
=
\left(K_{X\ast}^{(0)}\right)_{:,j}.
\end{displaymath}
For a transformed test configuration $g\star\bar\bx_j^\ast$, invariance of the projection gives
\begin{displaymath}
\PiA(g\star\bar\bx_j^\ast)=\PiA(\bar\bx_j^\ast).
\end{displaymath}
The transformed configuration therefore produces exactly the same scalar right-hand side. Moreover, $B_{\btheta}$ and its pivoted-Cholesky preconditioner depend only on the projected training configurations and are unchanged by transformations of the test inputs. Every covariance solve is initialized at zero and uses the same deterministic preconditioner and stopping rule. Consequently, the finite PCG iterates for a test configuration and any of its transformed versions coincide in exact arithmetic. Processing the right-hand sides in blocks of $25$ only evaluates several independent PCG recurrences simultaneously and does not couple or alter them.
\end{rem}

\subsubsection{Computational environment and recorded times}
The full-GP and SVGP experiments are run in R~4.4.2 on one CPU core, with R-level parallelization disabled. PCG inference is run in double precision on one NVIDIA H100 GPU using CUDA~12.2 and PyKeOps, together with four CPU cores. For each method, elapsed wall-clock time is recorded separately for model fitting, posterior-mean prediction, and predictive-covariance and LogS evaluation. For PCG, model-fitting time comprises parameter estimation, preconditioner construction, and the final training solve. Peak memory is measured by Slurm's maximum resident set size, \texttt{MaxRSS}, and reported in GB.

All code used to generate the experiments is available at \href{https://anonymous.4open.science/r/Scalable-Equivariant-Gaussian-Processes-6F6F/README.md}{GitHub}.
\begin{figure}[tbp]
    \centering
    \includegraphics[width=\linewidth]{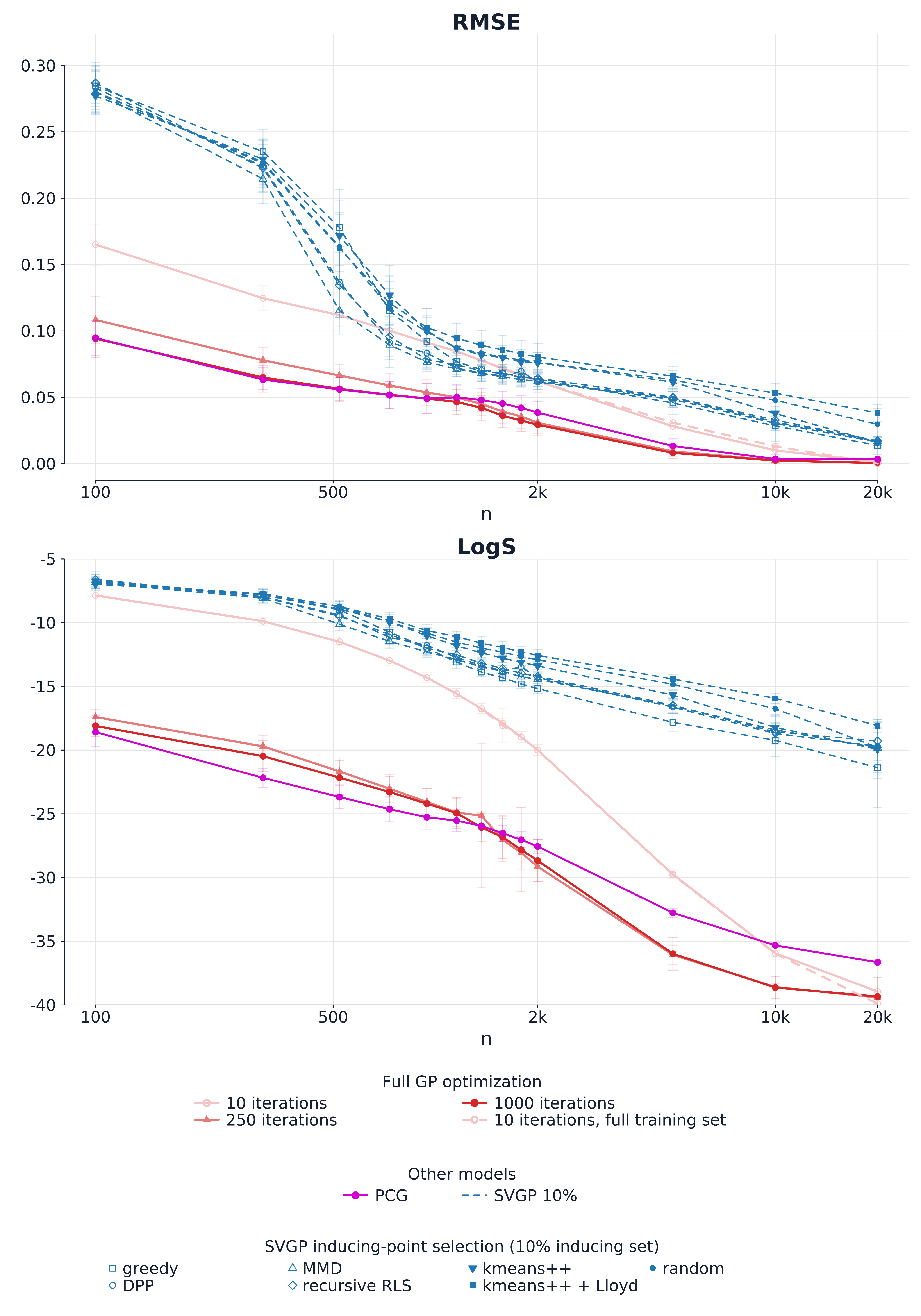}
    \caption{RMSE and LogS learning curves comparing SVGP variants, PCG-based matrix-free full-GP inference, and dense full-GP baselines across increasing training-set sizes.}
    \Description{Two vertically arranged learning-curve panels. The upper panel shows RMSE and the lower panel shows LogS against training-set size for dense full-GP fits with several optimization budgets, matrix-free PCG inference, and SVGP models using several inducing-point allocation strategies.}
    \label{fig:lc-pcg}
\end{figure}

\begin{figure}[tbp]
    \centering
    \includegraphics[width=\linewidth]{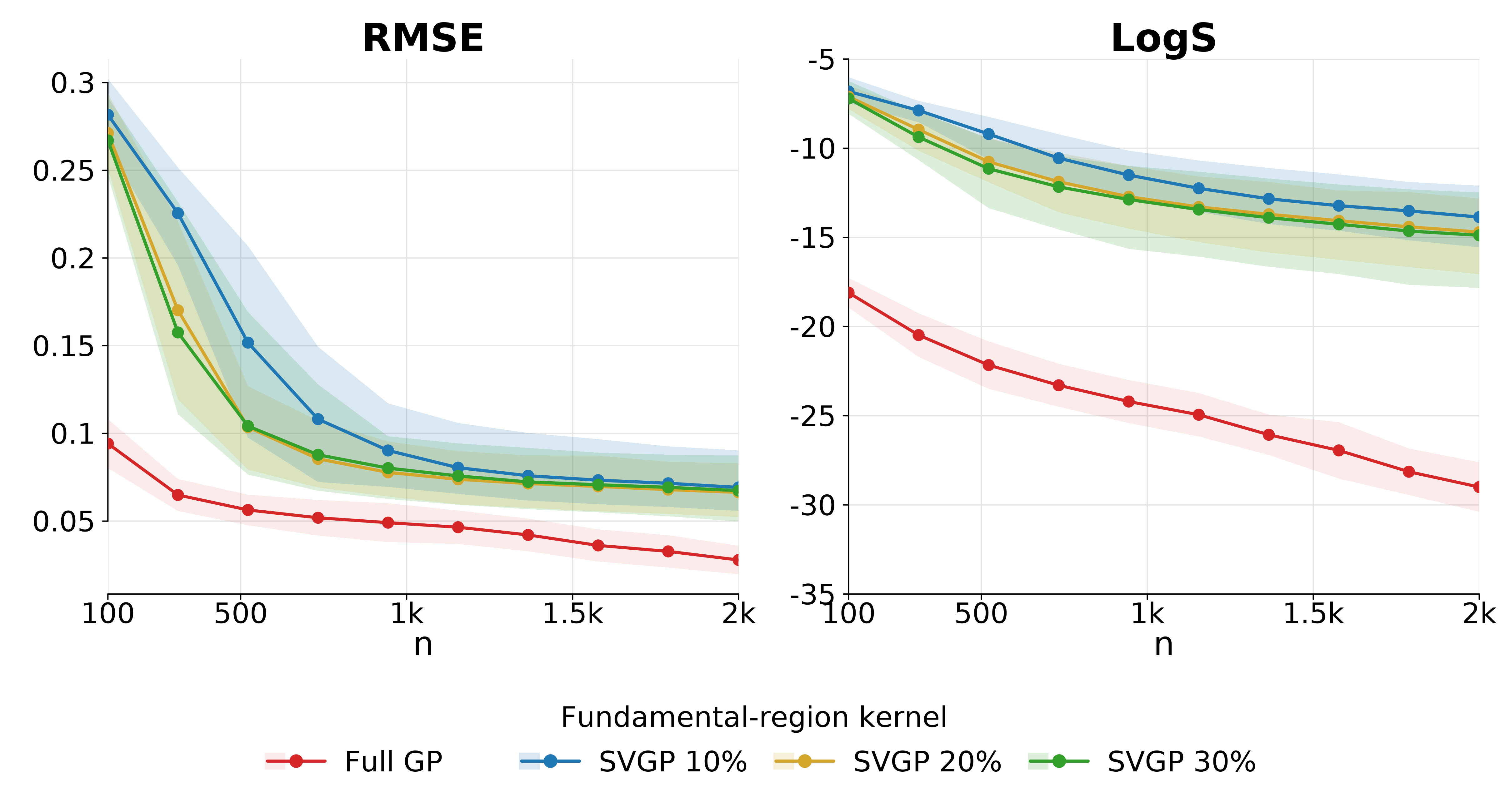}
    \caption{Learning curves for training sizes $n\leq 2000$. Shaded areas represent the range of the mean score plus or minus one standard deviation over all inducing-point allocation schemes for SVGP inducing-set ratios $r\in\{0.1,0.2,0.3\}$.}
    \Description{Two learning-curve panels show RMSE and LogS against training-set size. A dense full-GP curve is compared with three translucent SVGP envelopes corresponding to inducing-set ratios of 10, 20, and 30 percent.}
    \label{fig:envelope plot}
\end{figure}

\subsection{Results}
\label{subsec:nmf-results}

Figure~\ref{fig:lc-pcg} presents the learning curves for training-set sizes between $100$ and $20000$, comparing the equivariant full GP, PCG-based matrix-free inference, and the SVGP with inducing-size ratio $r:=|Z|/n=0.1$ across the different inducing-point allocation schemes. The PCG posterior mean closely tracks that of the full GP trained with the largest hyperparameter-optimization budget, and the two methods attain similar RMSE values over most training sizes. The difference is more pronounced for the LogS, which is more sensitive to errors in both the posterior mean and the complete joint posterior covariance. Nevertheless, PCG consistently achieves substantially lower LogS values than the SVGP variants. This indicates that the frame-reduced PCG implementation retains much of the distributional accuracy of the dense full GP, in addition to providing accurate point predictions. The remaining differences between the dense and PCG curves can arise from their separately estimated hyperparameters and from the finite tolerances of the iterative training and covariance solves. Tightening these tolerances may further reduce the discrepancy, at the cost of additional computation. The full-GP curves also demonstrate the importance of sufficiently optimizing the kernel hyperparameters. Increasing the Adam budget from $10$ to $250$ or $1000$ iterations generally improves both RMSE and LogS. Capping the hyperparameter-estimation set at $1000$ configurations makes these longer optimization runs feasible throughout the learning curve, after which the resulting parameters are used for prediction with all $n$ training observations. The comparison with the ten-iteration fit using the complete training set shows that increasing the amount of data used in each likelihood evaluation does not compensate consistently for an insufficient optimization budget. We therefore interpret the optimizer budget as an important component of the full-GP baseline, without concluding that predictive performance is independent of the training-set size. Regarding the SVGP, the kernel-dependent inducing-point allocation schemes generally provide better predictive distributions than the kernel-free schemes in the small and intermediate training regimes. MMD in Algorithm~\ref{alg:mmd-refine} and M-DPP in Algorithm~\ref{alg:mdpp} attain particularly strong scores below $n=2000$, whereas the greedy method in Algorithm~\ref{alg:greedy} becomes competitive and often gives the best SVGP scores for larger training sets. At the largest training sizes, the differences among the allocation schemes become smaller, and $k$-means++ and random selection become increasingly competitive. This suggests that the data contain increasing redundancy relative to the inducing-set size, so that the benefit of the more involved kernel-dependent allocation criteria eventually diminishes. Figure~\ref{fig:envelope plot} summarizes the effect of increasing the inducing-size ratio from $r=0.1$ to $r\in\{0.2,0.3\}$. Larger inducing sets improve both RMSE and LogS, but the improvement is concentrated mainly in the smallest training regimes, particularly for $n\leq500$. Beyond this regime the envelopes largely overlap, indicating diminishing predictive returns from increasing $M$, while the computational cost continues to grow. Figure~\ref{fig:lcsmallregime3rows} reports the individual learning curves for these ratios and also illustrates the importance of encoding the structural constraint directly in the kernel. The models using the standard base kernel fail to improve meaningfully with increasing $n$, for both the full GP and the SVGP, whereas the fundamental-region kernel $K_\Pi$ yields steadily improving predictions. This provides direct empirical evidence that scalability alone is insufficient when the rotational structure of the dipole response is not represented in the covariance model. The minibatch comparison in Figure~\ref{fig:boxplotssvgp20k} further shows that the variational training configuration affects the final SVGP accuracy. At $n=20000$, increasing the minibatch size from $500$ to $1000$ or $4000$ generally improves RMSE and LogS, although the magnitude of this effect depends on the inducing-point allocation method. The improvement comes at the expected increase in computational cost. We consequently use a minibatch containing $20\%$ of the non-inducing training observations for the main learning-curve comparison.

The computational results in Figure~\ref{fig:computation times} are consistent with the theoretical scaling of the three approaches. With the output dimension fixed at three, dense full-GP inference requires $\mathcal{O}(n^3)$ operations and $\mathcal{O}(n^2)$ memory. An SVGP update with minibatch size $B$ and inducing-set size $M$ requires $\mathcal{O}(BM^2+M^3)$ operations and $\mathcal{O}(BM+M^2)$ memory. Since the experiments use $M=\lfloor rn\rfloor$ and $B=0.2(n-M)$, fixed inducing ratios retain cubic time and quadratic memory scaling asymptotically, but with constants that depend strongly on $r$, which explains the increasing separation between the $10\%$, $20\%$, and $30\%$ SVGP curves. For the large-$n$ PCG branch, let $T_{\mathrm{CG}}$ denote the number of iterations and $R$ the pivoted-Cholesky preconditioner rank. Matrix-free training requires $\mathcal{O}(T_{\mathrm{CG}}n^2+nR^2+R^3)$ operations and $\mathcal{O}(nR+R^2)$ memory, in addition to the bounded dense hyperparameter fit on at most $2000$ observations. For fixed $R$ and bounded iteration count, this is quadratic in time and linear in memory with respect to $n$. Posterior-mean evaluation after the training solve requires $\mathcal{O}(nn_\ast)$ operations. Evaluation of the full joint LogS is more expensive because the $n_\ast$ right-hand sides in \eqref{eq:pcg covariance solves} must additionally be solved, even though they are processed in memory-efficient blocks. Accordingly, PCG requires the least time for model fitting and posterior-mean evaluation and exhibits the smallest, nearly constant peak-memory footprint in Figure~\ref{fig:computation times}. Its full predictive-covariance and LogS computation grows more strongly and, for the largest training sets, is slower than predictive inference with the $10\%$ SVGP. The dense full GP shows the steepest increase in predictive-inference time and memory, whereas increasing the SVGP inducing ratio improves accuracy at a substantial computational cost. These wall-clock comparisons reflect both the algorithms and the computational environments specified in Section~\ref{subsec:implementation}, as the full GP and SVGP are evaluated on one CPU core, while PCG uses an NVIDIA H100 GPU. The operation and memory trends nevertheless agree with the theoretical scaling and demonstrate the complementary benefits of SVGP and PCG. The SVGP provides inexpensive approximate uncertainty quantification, while PCG provides full-GP-quality posterior means and joint predictive distributions with substantially reduced memory requirements.
\begin{figure}[H]
    \centering
    \includegraphics[width=\linewidth]{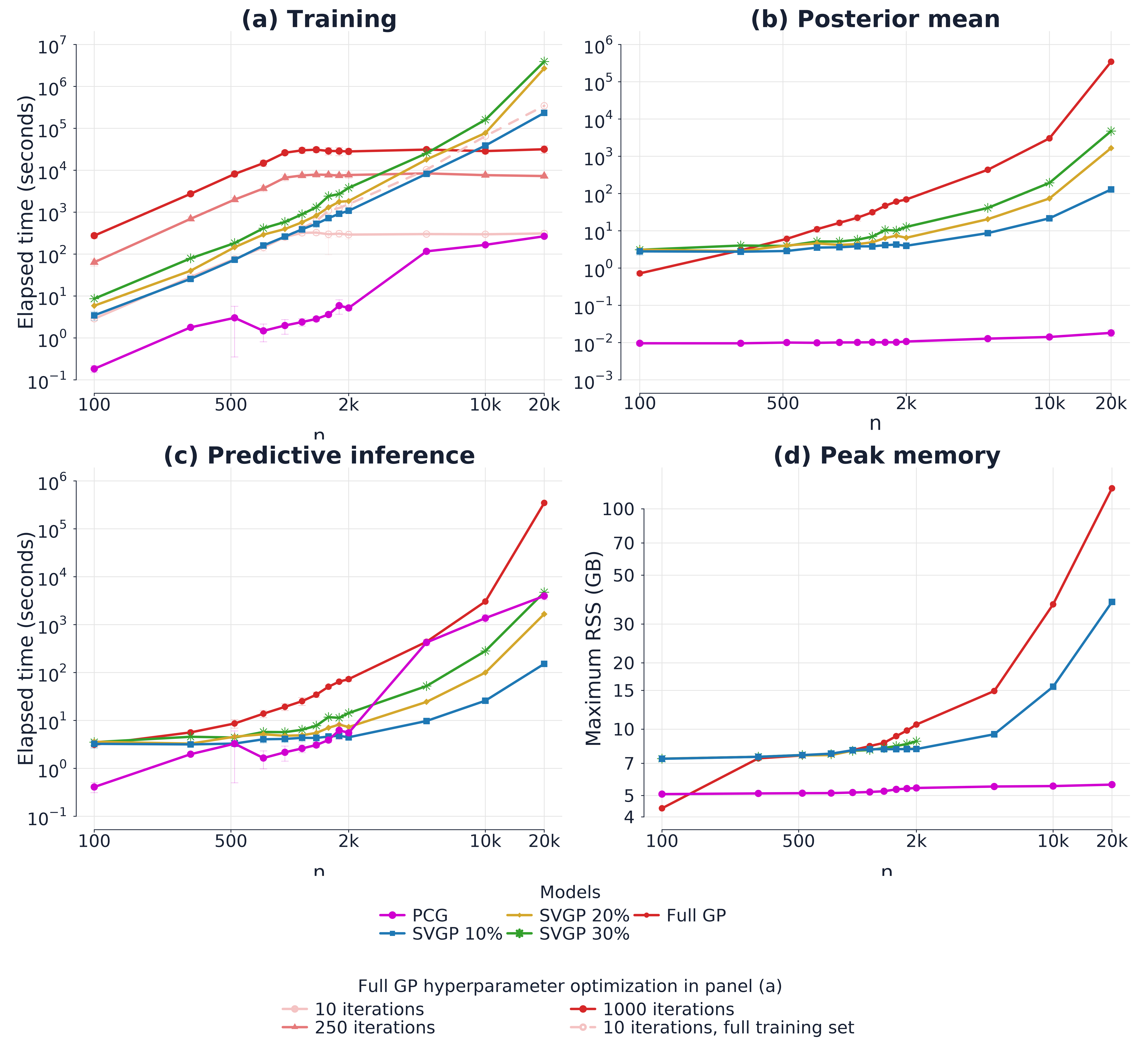}
    \caption{Wall-clock computation times and peak resident memory for the different GP models.}
    \Description{Four panels compare computational scaling against training-set size. They show fitting time, posterior-mean prediction time, predictive-inference time including covariance and LogS evaluation, and peak resident memory for the full GP, PCG, and SVGP models.}
    \label{fig:computation times}
\end{figure}

\section{Conclusion and Outlook}
\label{sec:discussion}
This work is guided by three main requirements for scalable equivariant GP modeling. Equivariant kernels must be fast to evaluate, equivariance must be preserved under scalable inference, and molecular kernel constructions must remain applicable as the dimension and number of atoms increase. Theorem \ref{thm:equivariant-conditioning} addresses the second requirement by establishing that stochastic equivariance is preserved under conditioning. The structure is therefore inherited by the full posterior distribution, including its mean, covariance, and dependencies between predictions. Theorem \ref{thm:finite response equiv} further establishes that equivariance of the conditional GP at a single input for finitely many sufficiently informative response values implies stochastic equivariance of the GP throughout its domain, thereby connecting finitely many local conditional statements to a global distributional property of the GP. Together, these results provide the theoretical basis for preserving stochastic equivariance in sparse and scalable GP inference. While the theoretical results apply to general equivariant kernels, in our experiments, integration-free fundamental-region kernels provide a computationally efficient realization of this framework. They avoid numerical group integration while preserving positive definiteness and equivariance and can be constructed for any fixed molecular type and number of atoms. Their particular structure also enables an additional simplification in the NMF application. Combining PCG with the scalar diagonal base kernel and orthogonal output transformations yields an exact frame reduction of the vector-valued kernel system to a scalar kernel system with three right-hand sides. This retains the posterior of the equivariant full GP while providing the computational and memory requirements of matrix-free PCG with the underlying base kernel, up to the three output coordinates. The frame reduction also makes the full joint posterior covariance feasible to evaluate through scalar-frame solves, although this remains substantially more expensive than posterior-mean prediction. This reduction is specific to the algebraic structure of the kernel used in this application, whereas the stochastic equivariance results hold for equivariant kernels more generally. PCG and SVGP provide complementary approaches to scalable inference. PCG replaces dense matrix factorization by matrix-free iterative solves and converges to the posterior of the full GP as the numerical tolerance is tightened. The equivariant SVGP provides a general sparse approximation that also can be combined with different equivariant kernels. Its inducing-set size and variational representation allow predictive accuracy and computational cost to be balanced while preserving stochastic equivariance of the resulting approximate posterior. The SVGP provides its full approximate predictive distribution, including the joint posterior covariance, in closed form through the inducing representation. This makes uncertainty quantification feasible at scale while preserving stochastic equivariance of the complete approximate posterior distribution. Its variational approximation can be less faithful to the full GP than a sufficiently converged PCG solve. The $\mathrm{SO}(2)$ experiments illustrate that equivariant SVGP posteriors respect the required transformation law even under strong sparsification. The approximation primarily affects predictive accuracy and the calibration and magnitude of posterior uncertainty, while the structural constraint is retained. The NMF experiments show that PCG and SVGP substantially reduce the computational requirements of equivariant GP inference. The PCG closely approximates the full GP with considerably lower memory usage, while the SVGP offers additional control through the inducing budget and the inducing-point allocation scheme. The two approaches consequently provide complementary routes to full equivariant predictive distributions. PCG retains high fidelity to the full GP and supports full uncertainty quantification through additional iterative solves, while SVGP gives direct access to approximate joint posterior uncertainty at lower computational cost. The observed saturation of the SVGP LogS indicates that the variational family and the inducing representation of the posterior covariance remain important bottlenecks. A limitation of the fundamental-region kernel used in the experiments is its potential discontinuity at the boundaries of the fundamental region. This complicates the usual joint optimization of inducing locations through the ELBO and motivates the discrete allocation schemes considered in this work, which provide better accuracy when combined with kernel-induced distances reflecting the equivariant structure. Future research could combine the general equivariant SVGP framework with smoother equivariant kernels, including potential-based constructions such as MOB-GP \citep{Sun_2022}. More expressive invariant base kernels and learned invariant or equivariant representations could further improve predictive accuracy.

Another important direction is the extension from datasets containing configurations of one fixed molecular type to chemically heterogeneous datasets such as QM9. Such data require a common representation for molecules with different numbers and types of atoms together with the corresponding permutation and rigid-motion symmetries. Equivariant neural networks could be used to learn suitable representations for equivariant GPs, combining flexible representation learning with probabilistic predictions and calibrated uncertainty. Equivariance also suggests applications in active learning and experimental design. Since observations related by group actions share structural information, symmetry-aware acquisition strategies may reduce the number of expensive evaluations required to learn a physical system. Vector-valued and operator-valued kernels arise naturally in PDE-constrained and multi-physics problems, where equivariant GP constructions could provide symmetry-consistent probabilistic surrogate models. Overall, the results indicate that equivariant Gaussian processes can provide symmetry-consistent and uncertainty-aware predictions at scales that were previously difficult to reach with GPs. Encoding symmetry through equivariant kernels and preserving it throughout conditioning and scalable inference provides a principled way to propagate physical structure through the full Bayesian pipeline. The resulting frameworks make full equivariant predictive distributions and feasible uncertainty quantification available in computationally demanding regimes.

\section*{Acknowledgments}
We acknowledge the support of the Digitization Commission (DigiK) of the University of Bern through the project ``Perception in Statistics, Econometrics and Probability.'' Calculations were performed on UBELIX (\url{https://www.id.unibe.ch/hpc}), the HPC cluster at the University of Bern. We thank August Lykke M{\o}ller and Ove Christiansen for providing the N-methylformamide data, and Oliver Warth for suggesting the MMD-based inducing-point selection scheme.

OpenAI GPT-5.6 Sol was used as a language and \LaTeX{} editing aid, for brainstorming possible formulations and implementation strategies, and for refactoring and debugging \texttt{R} and Slurm code for the HPC environment. The model was not used to conduct experiments, generate data, perform analysis, or originate the core scientific ideas underlying the work.
All suggestions were critically reviewed, independently implemented where relevant, and verified by the authors, who retain full responsibility for the algorithms, experimental design, results, interpretations, and final manuscript.

\bibliographystyle{abbrvnat}

\bibliography{references}

\appendix

\section{Effect of Different SVGP Batch Sizes}

\begin{figure}[H]
    \centering
    \includegraphics[width=\linewidth]{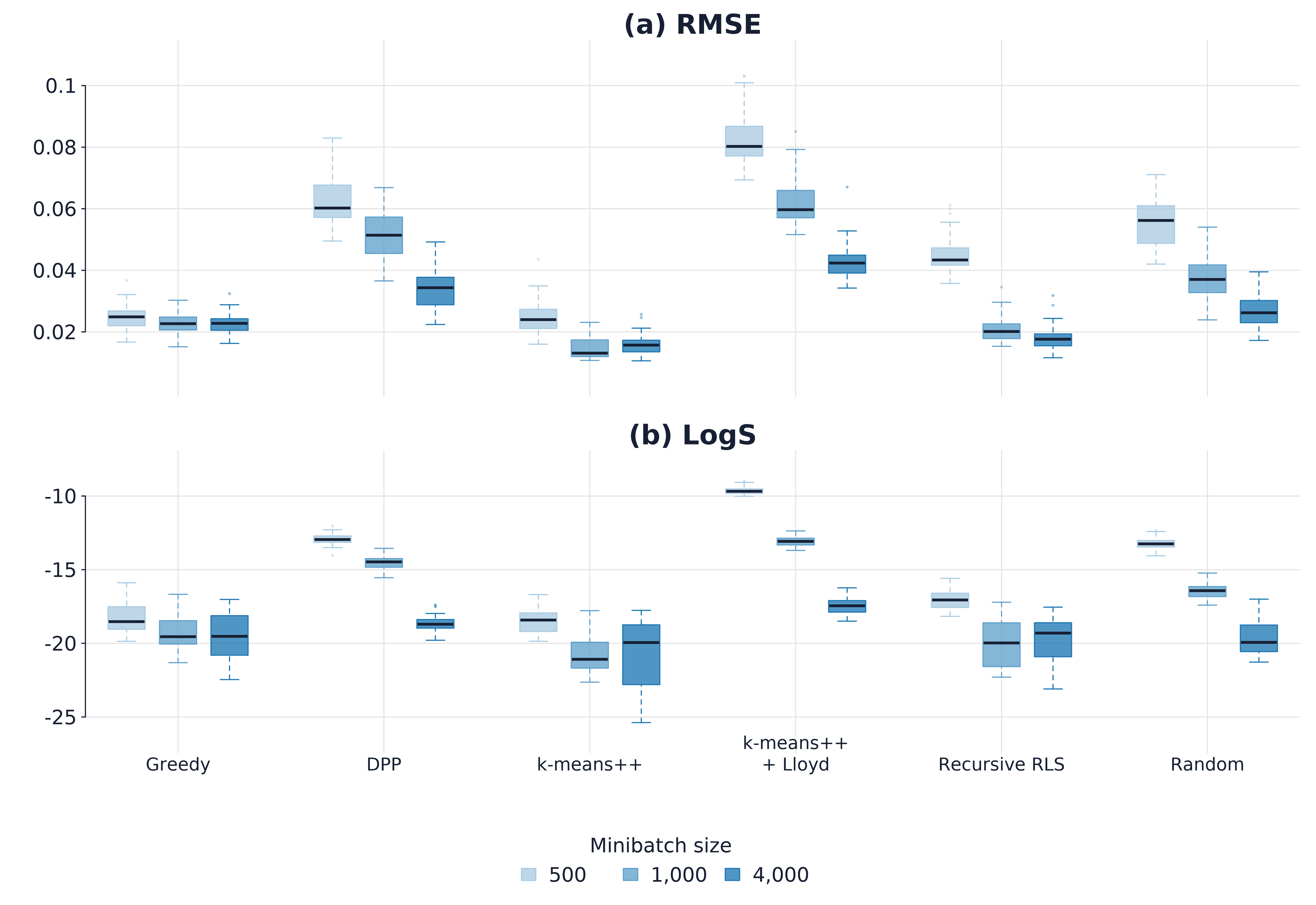}
    \caption{Boxplots comparing SVGP predictive performance at training size $n=20000$ for several minibatch ratios and inducing-point selection strategies.}
\Description{Boxplots comparing RMSE and log-score distributions for SVGP runs at training size 20000 under different minibatch ratios and inducing-point selection methods.}
    \label{fig:boxplotssvgp20k}
\end{figure}

\section{MMD-Based Inducing-Point Allocation}\label{subsec:MMD-comp}
Let $K:D\times D\to\mathbb R^{p\times p}$ be a positive-definite (PD) matrix-valued kernel. By vector-valued RKHS theory, there exists a separable Hilbert space $\mathcal U$ and a Hilbert–Schmidt feature operator $\Psi:D\to \mathrm{HS}(\mathbb R^p,\mathcal U)$ such that
\begin{equation}K(\bx,\bx')=\Psi(\bx)^\ast \Psi(\bx'),\qquad \langle \Psi(\bx),\Psi(\bx')\rangle_{\mathrm{HS}}=\operatorname{tr}\!\big(\Psi(\bx)^\ast \Psi(\bx')\big)=\operatorname{tr}K(\bx,\bx'). \label{eq:vv-feature}\end{equation}
For empirical input measures $\hat P_X=\frac1n\sum_{i=1}^n\delta_{\bx_i}$ and $\hat P_Z=\frac1m\sum_{j=1}^m\delta_{\bz_j}$, define the vector-valued RKHS {mean operators}
\begin{displaymath}\mu_X:=\frac1n\sum_{i=1}^n \Psi(\bx_i)\in \mathrm{HS}(\mathbb R^p,\mathcal U), \qquad \mu_Z:=\frac1m\sum_{j=1}^m \Psi(\bz_j)\in \mathrm{HS}(\mathbb R^p,\mathcal U).\end{displaymath}
We define the matrix-valued MMD as the squared Hilbert–Schmidt distance
\begin{equation}\mathrm{MMD}^2_K(X,Z):=\|\mu_X-\mu_Z\|_{\mathrm{HS}}^2. \label{eq:mmd-vvrkhs}\end{equation}
Expanding \eqref{eq:mmd-vvrkhs} with \eqref{eq:vv-feature} yields
\begin{align}
\mathrm{MMD}^2_K(X,Z)
&= \frac{1}{n^2}\!\sum_{i,i'} \langle \Psi(\bx_i),\Psi(\bx_{i'})\rangle_{\mathrm{HS}}
 + \frac{1}{m^2}\!\sum_{j,j'} \langle \Psi(\bz_j),\Psi(\bz_{j'})\rangle_{\mathrm{HS}}
 - \frac{2}{nm}\!\sum_{i,j} \langle \Psi(\bx_i),\Psi(\bz_j)\rangle_{\mathrm{HS}} \nonumber\\
&= \frac{1}{n^2}\!\sum_{i,i'} \operatorname{tr}K(\bx_i,\bx_{i'})
 + \frac{1}{m^2}\!\sum_{j,j'} \operatorname{tr}K(\bz_j,\bz_{j'})
 - \frac{2}{nm}\!\sum_{i,j} \operatorname{tr}K(\bx_i,\bz_j).
\label{eq:mmd-trace}
\end{align}
{Thus the scalar contraction that enters the MMD is $\operatorname{tr}K(\cdot,\cdot)$ by construction.} You can equivalently view \eqref{eq:mmd-trace} as a scalar MMD computed with the (automatically PD) {trace kernel} $k_T(\bx,\bx'):=\operatorname{tr}K(\bx,\bx')$.

\begin{align*}
\mathrm{MMD}^2_K(X,Z)
&=\frac{1}{n(n-1)}
  \sum_{\substack{i,i'=1\\ i\neq i'}}^{n}
  \operatorname{tr}K(\bx_i,\bx_{i'}) \\
&\quad+\frac{1}{m(m-1)}
  \sum_{\substack{j,j'=1\\ j\neq j'}}^{m}
  \operatorname{tr}K(\bz_j,\bz_{j'}) \\
&\quad-\frac{2}{nm}
  \sum_{i=1}^{n}\sum_{j=1}^{m}
  \operatorname{tr}K(\bx_i,\bz_j).
\end{align*}

\noindent When replacing one inducing point $\bz_r$ by a new point $\bz_{\text{new}}$, only the second and third sums change. We can write the updated MMD as

\begin{align*}
\mathrm{MMD}^2_K(X,Z')
&= C_X
 + \frac{1}{m(m-1)}
   \!\left[
     \sum_{\substack{j,j'=1\\ j,j'\neq r}}^{m}
       \mathrm{tr}\big[K(\bz_j,\bz_{j'})\big]
     + 2\!\!\sum_{\substack{j=1\\ j\neq r}}^{m}
       \mathrm{tr}\big[K(\bz_j,\bz_{\text{new}})\big]
   \right] \\
&\quad
 - \frac{2}{nm}
   \!\left[
     \sum_{i=1}^{n}\!\!\sum_{\substack{j=1\\ j\neq r}}^{m}\!
       \mathrm{tr}\big[K(\bx_i,\bz_j)\big]
     + \sum_{i=1}^{n}\!\mathrm{tr}\big[K(\bx_i,\bz_{\text{new}})\big]
   \right],
\end{align*}

\noindent where $C_X$ collects all constant terms that depend only on $X$ and thus cancel in the MMD difference. Define
\begin{equation*}
a_r(\bz):=
\frac{1}{m-1}\sum_{j\neq r}\operatorname{tr}K(\bz_j,\bz)
-\frac{1}{n}\sum_{i=1}^{n}\operatorname{tr}K(\bx_i,\bz).
\end{equation*}
The incremental change can then be computed compactly as
\begin{equation*}
\Delta:=\mathrm{MMD}^2_K(X,Z')-\mathrm{MMD}^2_K(X,Z)
=\frac{2}{m}\left[a_r(\bz_{\mathrm{new}})-a_r(\bz_r)\right].
\end{equation*}

\noindent This form allows efficient local updates: only $\mathcal{O}(m+n)$ kernel evaluations are required to test the effect of swapping $\bz_r$ with $\bz_{\text{new}}$, instead of recomputing all pairwise terms.

\begin{figure}[tbp]
    \centering
    \includegraphics[width=1\linewidth]{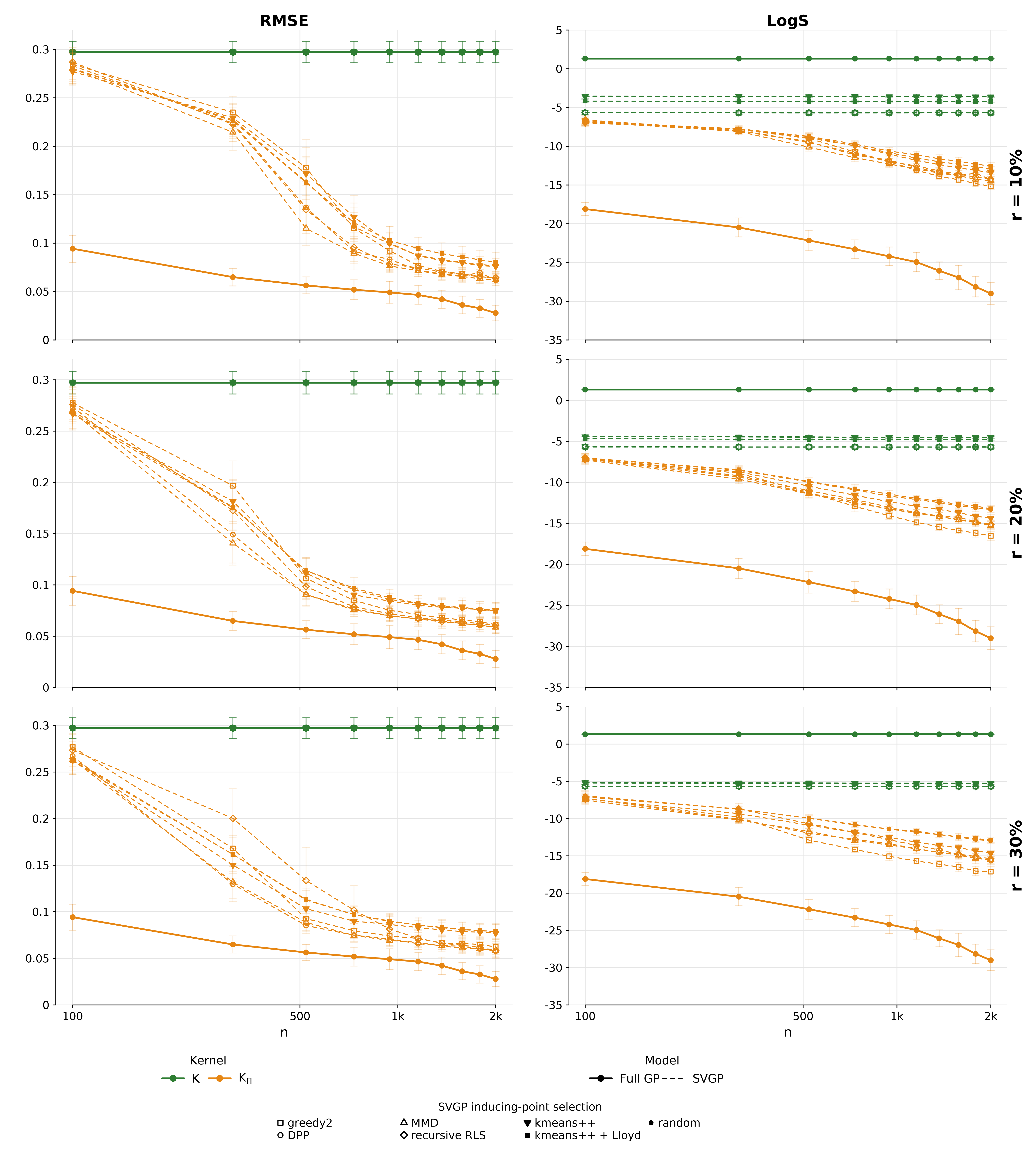}
    \caption{Extended SVGP learning curves in the small-data regime, showing three rows of metrics across inducing-point allocation methods and inducing-set ratios.}
    \Description{An extended appendix version of the small-data SVGP learning-curve figure, with three rows of panels comparing inducing-point selection methods across training sizes and inducing-set ratios.}
    \label{fig:lcsmallregime3rows}
\end{figure}

\begin{figure}[tbp]
    \centering
    \includegraphics[width=\linewidth]{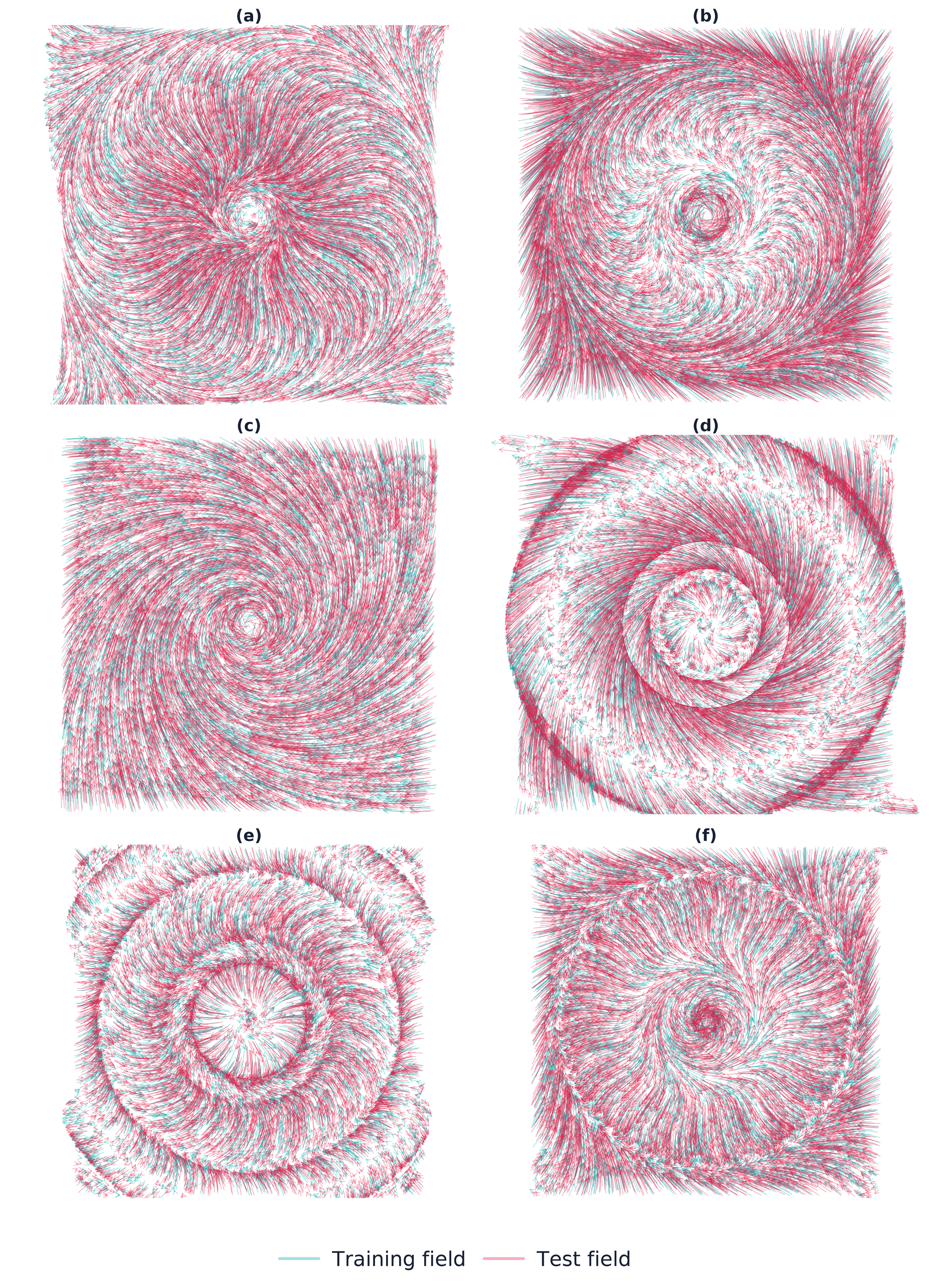}
    \caption{Further prior samples of $\mathrm{GP}(\boldsymbol{0},\Kpi)$, with hyperparameters sampled uniformly over $[0,20]^4$.}
    \Description{Six prior samples arranged in a three-by-two grid. Each panel overlays training and test vectors from a sampled SO(2)-equivariant vector field.}
    \label{fig:extra prior samples}
\end{figure}

\end{document}